\def\year{2021}\relax
\documentclass[letterpaper]{article} 
\usepackage{aaai21}  
\usepackage{times}  
\usepackage{helvet} 
\usepackage{courier}  
\usepackage[hyphens]{url}  
\usepackage{graphicx} 
\usepackage{natbib}  
\usepackage{caption} 
\usepackage{amsmath}
\usepackage{booktabs}
\usepackage{algorithm}
\usepackage{algorithmic}
\usepackage{multirow}
\usepackage{bbm}
\usepackage{amsmath,amssymb,mathrsfs}
\usepackage{graphicx} 
\usepackage{float} 
\usepackage{subfigure} 
\usepackage{amsthm}
\usepackage[switch]{lineno} 
\usepackage{multicol}
\usepackage[title]{appendix}
\usepackage{lipsum} 
\usepackage{hyperref}

\title{Action Candidate Based Clipped Double Q-learning for Discrete and Continuous Action Tasks}
\author{
	Haobo Jiang, Jin Xie$^{*}$, Jian Yang$^{*}$ \\
}
\affiliations{
	PCA Lab, Key Lab of Intelligent Perception and Systems for High-Dimensional Information of Ministry of Education\\
	Jiangsu Key Lab of Image and Video Understanding for Social Security\\
	School of Computer Science and Engineering, Nanjing University of Science and Technology, Nanjing, China \\
	\{jiang.hao.bo, csjxie, csjyang\}@njust.edu.cn
}

\begin{document}
	
\newtheorem{theorem}{Theorem}
\newtheorem{lemma}{Lemma}
\newtheorem{property}{Property}

\maketitle

\begin{abstract}
Double Q-learning is a popular reinforcement learning algorithm in Markov decision process (MDP) problems. 
Clipped Double Q-learning, as an effective variant of Double Q-learning, employs the clipped double estimator to approximate the maximum expected action value.
Due to the underestimation bias of the clipped double estimator, performance of clipped Double Q-learning may be degraded in some stochastic environments. 
In this paper, in order to reduce the underestimation bias, we propose an action candidate based clipped double estimator for Double Q-learning.
Specifically, we first select a set of elite action candidates with the high action values from one set of estimators.
Then, among these candidates, we choose the highest valued action from the other set of estimators.
Finally, we use the maximum value in the second set of estimators to clip the action value of the chosen action in the first set of estimators and the clipped value is used for approximating the maximum expected action value. 
Theoretically, the underestimation bias in our clipped Double Q-learning decays monotonically as the number of the action candidates decreases. 
Moreover, the number of action candidates controls the trade-off between the overestimation and underestimation biases.
In addition, we also extend our clipped Double Q-learning to continuous action tasks via approximating the elite continuous action candidates.
We empirically verify that our algorithm can more accurately estimate the maximum expected action value on some toy environments and yield good performance on several benchmark problems. All code and hyperparameters available at \href{https://github.com/Jiang-HB/AC_CDQ}{https://github.com/Jiang-HB/AC\_CDQ}.
\end{abstract}

\section{Introduction}
In recent years, reinforcement learning has achieved more and more attention.
It aims to learn an optimal policy so that cumulative rewards can be maximized via trial-and-error in an unknown environment \cite{sutton2018reinforcement}.
Q-learning \cite{watkins1992q} is one of widely studied reinforcement learning algorithms.
As a model-free reinforcement learning algorithm, it generates the optimal policy via selecting the action which owns the largest estimated action value.
In each update, Q-learning executes the maximization operation over action values for constructing the target value of Q-function.
Unfortunately, this maximization operator tends to overestimate the action values.
Due to the large positive bias, it is difficult to learn the high-quality policy for the Q-learning in many tasks \cite{thrun1993issues,szita2008many,strehl2009reinforcement}.
Moreover, such overestimation bias also exists in a variety of variants of Q-learning such as fitted Q-iteration \cite{strehl2006pac}, delayed Q-learning \cite{ernst2005tree} and deep Q-network (DQN) \cite{mnih2015human}.

Recently, several improved Q-learning methods have been proposed to reduce the overestimation bias.
Bias-corrected Q-learning \cite{lee2013bias} adds a bias correction term on the target value  so that the overestimation error can be reduced.
Softmax Q-learning \cite{song2019revisiting} and Weighted Q-learning \cite{d2016estimating} are proposed to soften the maximum operation via replacing it with the sum of the weighted action values.
The softmax operation and Gaussian approximation are employed to generate the weights, respectively. 
In Averaged Q-learning \cite{anschel2017averaged} and Maxmin Q-learning \cite{lan2020maxmin}, their target values are constructed to reduce the bias and variance via combining multiple Q-functions.

Double Q-learning \cite{hasselt2010double,van2013estimating,zhang2017weighted} is another popular method to avoid the overestimation bias.
In Double Q-learning, it exploits the online collected experience sample to randomly update one of two Q-functions.
In each update, the first Q-function selects the greedy action and the second Q-function evaluates its value.
Although Double Q-learning can effectively relieve the overestimation bias in Q-learning in terms of the expected value, its target value may occasionally be with the large overestimation bias during training process. 
To avoid it, clipped Double Q-learning \cite{fujimoto2018addressing} directly uses the maximum action value of one Q-function to clip the target value of the Double Q-learning.
Clipping Double Q-learning can be viewed as using the clipped double estimator to approximate the maximum expected value.
However, the clipped double estimator suffers from the large underestimation bias.

In order to reduce the large negative bias of the clipped double estimator, in this paper, we propose an action candidate based clipped double estimator for Double Q-learning.
Specifically, we first select a set of action candidates corresponding to high action values in one set of estimators. 
Then, among these action candidates, we choose the action with the highest value in the other set of estimators. 
At last, the corresponding action value of the selected action in the first set of estimators clipped by the maximum value in the second set of estimators is used to approximate the maximum expected value.
Actually, in clipped Double Q-learning, the selected action from one Q-function is independent of the action evaluation in the other Q-function. Thus, the selected action may correspond to the low action value in the second Q-function, which results in the large underestimation.
Through bridging the gap between the action selection and action evaluation from both Q-functions, our action candidate based clipped Double Q-learning can effectively reduce the underestimation bias. 
Theoretically, the underestimation bias in our clipped Double Q-learning decays monotonically as the number of action candidates decreases. 
Moreover, the number of action candidates can balance the overestimation bias in Q-learning and the underestimation bias in clipped Double Q-learning.
Furthermore, we extend our action candidate based clipped Double Q-learning to the deep version. Also, based on the action candidate based clipped double estimator, we propose an effective variant of TD3 \cite{fujimoto2018addressing} for the continuous action tasks.
Extensive experiments demonstrate that our algorithms can yield good performance on the benchmark problems.

\section{Background}
We model the reinforcement learning problem as an infinite-horizon discounted Markov Decision Process (MDP), which comprises a state space $\mathcal{S}$, a discrete action space $\mathcal{A}$, a state transition probability distribution $\mathcal{P}:\mathcal{S} \times \mathcal{A} \times \mathcal{S} \rightarrow \mathbb{R}$, an expected reward function $R:\mathcal{S}\times\mathcal{A}\rightarrow \mathbb{R}$ and a discount factor $\gamma \in [0,1]$.
At each step $t$, with a given state $s_t \in \mathcal{S}$, the agent receives a reward $r_t=R\left(s_t, a_t\right)$ and the new state $s_{t+1} \in \mathcal{S}$ after taking an action $a_t \in \mathcal{A}$.
The goal of the agent is to find a policy $\pi: \mathcal{S}\times\mathcal{A}\rightarrow [0, 1]$ that maximizes the expected return {${\mathbb{E}_\pi\left[\sum_{t=0}^{\infty}\gamma^t r_t\right]}$}.

In the MDP problem, the action value function ${Q^\pi(s,a)=\mathbb{E}_\pi\left[\sum_{t=0}^{\infty}\gamma^t r_t|s_0=s, a_0=a\right]}$ denotes the expected return after doing the action $a$ in the state $s$ with the policy $\pi$. 
The optimal policy can be obtained as: $\pi^*(s) = \arg \max_{a\in \mathcal{A}}Q^*(s,a)$ where the optimal action value function $Q^*(s,a)$ satisfies the {Bellman optimality equation}:
\begin{normalsize}
	\begin{equation} \label{policy}
	\begin{split}
	Q^*\left(s,a\right)=R\left(s,a\right) + \gamma \sum_{s'\in\mathcal{S}}\mathcal{P}_{sa}^{s'}\max_{a'\in\mathcal{A}}Q^*\left(s',a'\right).
	\end{split}
	\end{equation}
\end{normalsize}
\textbf{(Double) Q-learning.} To approximate $Q^*\left(s,a\right)$,  Q-learning constructs a Q-function and updates it in each step via $Q\left(s_t,a_t\right)\leftarrow Q\left(s_t,a_t\right)+ \alpha\left(y_t^{Q} - Q\left(s_t,a_t\right)\right)$, where the target value $y_t^{Q}$ is defined as below:
\begin{normalsize}
	\begin{equation} \label{policy} 
	\begin{split}
	y_t^{Q} = r_{t}+\gamma \max_{a' \in \mathcal{A}}Q\left(s_{t+1},a'\right).
	\end{split}
	\end{equation}
\end{normalsize}
Instead, Double Q-learning maintains two Q-functions, $Q^A$ and $Q^B$, and randomly updates one Q-function, such as $Q^A$, with the target value $y_t^{DQ}$ as below:
\begin{normalsize}
	\begin{equation} \label{policy}
	\begin{split}
	y_t^{DQ} = r_t + \gamma Q^B\left(s_{t+1}, \arg \max_{a'\in \mathcal{A}}  Q^A(s_{t+1},a')\right).
	\end{split}
	\end{equation}
\end{normalsize}
\textbf{Clipped Double Q-learning.}
It uses the maximum action value of one Q-function to clip the target value in Double Q-learning as below to update the Q-function:
 \begin{normalsize}
 	\begin{equation} \label{policy}
 	\begin{split}
 	y_t^{CDQ} = r_t + \gamma  \min\left\{Q^A(s_{t+1},a^*), Q^B(s_{t+1},a^*\right\},
 	\end{split}
 	\end{equation}
 \end{normalsize}
where $a^*=\arg\max_aQ^A\left(s_{t+1},a\right)$.

\textbf{Twin Delayed Deep Deterministic policy gradient (TD3).} 
TD3 applies the clipped Double Q-learning into the continuous action control with the actor-critic framework.
Specifically, it maintains a actor network $\mu\left(s; {\boldsymbol{\phi}}\right)$ and two critic networks $Q\left(s, a;\boldsymbol{\theta}_1\right)$ and $Q\left(s, a;\boldsymbol{\theta}_2\right)$. Two critic networks are updated via $\boldsymbol{\theta}_i \leftarrow \boldsymbol{\theta}_i + \alpha \nabla_{\boldsymbol{\theta}_i}\mathbb{E}\left[\left(Q\left(s_t, a_t;\boldsymbol{\theta}_i\right)-y_t^{TD3}\right)^2\right]$. The target value $y_t^{TD3}$ is defined as below:
\begin{normalsize}
\begin{equation} \label{policy}
y_t^{TD3} = r_t + \gamma \min_{i=1,2}Q\left(s_{t+1}, \mu\left(s_{t+1}; {\boldsymbol{\phi}^-}\right);\boldsymbol{\theta}_i^-\right),\end{equation}
\end{normalsize}
where $\boldsymbol{\phi}^-$ and $\boldsymbol{\theta}_i^-$ are the soft updated parameters of $\boldsymbol{\phi}$ and $\boldsymbol{\theta}_i$. The actor $\mu\left(s; {\boldsymbol{\phi}}\right)$ is updated via $\boldsymbol{\phi} \leftarrow \boldsymbol{\phi} + \alpha \nabla_{\boldsymbol{\phi}}J$, where the policy graident $\nabla_{\boldsymbol{\phi}}J$ is:
\begin{normalsize}
	\begin{equation} \label{policy}
	\begin{split}
	\nabla_{\boldsymbol{\phi}}J=\mathbb{E}\left[\left.\nabla_{a} Q\left(s_{t}, a ; \boldsymbol{\theta}_1\right)\right|_{a=\mu\left(s_{t} ; \boldsymbol{\phi}\right)} \nabla_{\boldsymbol{\phi}} \mu\left(s_{t} ; \boldsymbol{\phi}\right)\right].
	\end{split}
	\end{equation}
\end{normalsize}

 \begin{algorithm*}[ht]
	\caption{Action Candidate Based Clipped Double Q-learning}
	\label{alg:algorithm1}
	\textbf{Initialize} Q-functions $Q^A$ and $Q^B$, initial state $s$ and the number $K$ of action candiadte.
	\begin{algorithmic}[1] 
		\REPEAT
		\STATE Select action $a$ based on $Q^A\left(s,\cdot\right)$, $Q^B\left(s,\cdot\right)$ (e.g., $\epsilon$-greedy in $Q^A\left(s,\cdot\right)+Q^B\left(s,\cdot\right)$) and observe reward $r$, next state $s'$.
		\IF{update $Q^A$}
		\STATE Determine action candidates $\boldsymbol{\mathcal{M}}_K$ from $Q^B\left(s', \cdot\right)$ and define $a_K^* = \arg \max_{a \in \boldsymbol{\mathcal{M}}_K} Q^A\left(s', a\right)$.
		\STATE $Q^A\left(s,a\right) \leftarrow Q^A\left(s,a\right) + \alpha\left(s,a\right) \cdot \left( r+\gamma \min\left\{Q^B\left(s', a_K^*\right), {\max}_{a} Q^A\left(s', a\right)\right\} -Q^A\left(s,a\right)  \right)$.
		\ELSIF{update $Q^B$}
		\STATE Determine action candidates $\boldsymbol{\mathcal{M}}_K$ from $Q^A\left(s', \cdot\right)$ and define $a_K^* = \arg \max_{a \in \boldsymbol{\mathcal{M}}_K} Q^B\left(s', a\right)$.
		\STATE $Q^B\left(s,a\right) \leftarrow Q^B\left(s,a\right) + \alpha\left(s,a\right) \cdot \left( r+\gamma \min\left\{Q^A\left(s', a_K^*\right), {\max}_{a} Q^B\left(s', a\right)\right\} -Q^B\left(s,a\right)  \right)$.
		\ENDIF
		\STATE $s \leftarrow s'$
		\UNTIL{end}
	\end{algorithmic}
\end{algorithm*}

\section{Estimating the Maximum Expected Value}
\subsection{Revisiting the Clipped Double Estimator}
Suppose that there is a finite set of $N$ $\left(N\geq2\right)$ independent random variables $\boldsymbol{X}=\left\{X_1,\ldots,X_N\right\}$ with the expected values $\boldsymbol{\mu}=\left\{\mu_1,\mu_2,\ldots,\mu_N\right\}$. 
We consider the problem of approximating the maximum expected value of the variables in such a set: $\mu^*={\max}_{i}\mu_i={\max}_{i}\mathbb{E}\left[X_i\right]$.
The clipped double estimator \cite{fujimoto2018addressing} denoted as $\hat{\mu}_{CDE}^*$ is an effective estimator to estimate the maximum expected value.

Specifically, let $S=\bigcup_{i=1}^{N}S_i$ denote a set of samples, where $S_i$ is the subset containing samples for the variable $X_i$. 
We assume that the samples in $S_i$ are independent and identically distributed (i.i.d). 
Then, we can obtain a set of the unbiased estimators $\hat{\boldsymbol{\mu}}=\left\{\hat{\mu}_1,\hat{\mu}_2,\ldots,\hat{\mu}_N\right\}$ where each element $\hat{\mu}_i$ is a unbiased estimator of $\mathbb{E}\left[X_i\right]$ and can be obtained by calculating the sample average: $\mathbb{E}\left[X_i\right]\approx\hat{\mu}_i \stackrel{\text { def }}{=} \frac{1}{|S_i|}\sum_{s \in S_i}s$. 
Further, we randomly divide the set of samples $S$ into two subsets: $S^A$ and $S^B$. Analogously, two sets of unbiased estimators $\hat{\boldsymbol{\mu}}^A=\left\{\hat{\mu}_1^A,\hat{\mu}_2^A,\ldots,\hat{\mu}_N^A\right\}$ and $\hat{\boldsymbol{\mu}}^B=\left\{\hat{\mu}_1^B,\hat{\mu}_2^B,\ldots,\hat{\mu}_N^B\right\}$ can be obtained by sample average:
 $\hat{\mu}_i^A=\frac{1}{|S_i^A|}\sum_{s \in S_i^A}s$, $\hat{\mu}_i^B=\frac{1}{|S_i^B|}\sum_{s \in S_i^B}s$.
 Finally, the clipped double estimator combines $\hat{\boldsymbol{\mu}}$, $\hat{\boldsymbol{\mu}}^A$ and $\hat{\boldsymbol{\mu}}^B$ to construct the following estimator to approximate the maximum expected value:
 \begin{normalsize}
 	\begin{equation} \label{cde}
 	\begin{split}
 	\mu^*={\max}_{i}\mu_i\approx \min\left\{\hat{\mu}_{a^*}^B, \max_i\hat{\mu}_i\right\},
 	\end{split}
 	\end{equation}
 \end{normalsize}
where the variable $\max_i\hat{\mu}_i$ is called the single estimator denoted as $\hat{\mu}_{SE}^*$ and the variable $\hat{\mu}_{a^*}^B$ is called the double estimator denoted as $\hat{\mu}_{DE}^*$. 

For single estimator, it directly uses the maximum value in $\hat{\boldsymbol{\mu}}$ to approximate the maximum expected value.
Since the expected value of the single estimator is no less than $\mu^*$, the single estimator has overestimation bias.
Instead, for double estimator, it first calculates the index $a^*$ corresponding to the maximum value in $\hat{\boldsymbol{\mu}}^A$, that is $\hat{{\mu}}_{a^*}^A=\max_i \hat{{\mu}}_i^A$, and then uses the value $\hat{\mu}_{a^*}^B$ to estimate the maximum expected value. 
Due to the expected value of double estimator is no more than $\mu^*$, it is underestimated.

Although the double estimator is underestimated in terms of the expected value, it still can't entirely eliminate the overestimation \cite{fujimoto2018addressing}. By clipping the double estimator via single estimator, the clipped double estimator can effectively relieve it.
However, due to the expected value of $\min\left\{\hat{\mu}_{a^*}^B, \max_i\hat{\mu}_i\right\}$ is no more than that of $\hat{\mu}_{a^*}^B$, the clipped double estimator may further exacerbate the underestimation bias in the double estimator and thus suffer from larger underestimation bias. 

\begin{algorithm*}[ht]
	\caption{Action Candidate Based TD3}
	\label{alg:algorithm}
	Initialize critic networks $Q\left(\cdot;\boldsymbol{\theta}_1\right)$, $Q\left(\cdot;\boldsymbol{\theta}_2\right)$, and actor networks $\mu\left(\cdot; {\boldsymbol{\phi}_1}\right)$,  $\mu\left(\cdot; {\boldsymbol{\phi}_2}\right)$ with random parameters $\boldsymbol{\theta}_1$, $\boldsymbol{\theta}_2$, $\boldsymbol{\phi}_1$, $\boldsymbol{\phi}_2$  \\
	Initialize target networks $\boldsymbol{\theta}_1^- \leftarrow \boldsymbol{\theta}_1$, $\boldsymbol{\theta}_2^- \leftarrow \boldsymbol{\theta}_2$, $\boldsymbol{\phi}_1^- \leftarrow \boldsymbol{\phi}_1$, $\boldsymbol{\phi}_2^- \leftarrow \boldsymbol{\phi}_2$\\
	Initialize replay buffer $\mathcal{D}$
	\begin{algorithmic}[1] 
		\FOR{$t = 1:T$}
		\STATE Select action with exploration noise $a \sim \mu\left(s; {\boldsymbol{\phi}_1}\right)+\epsilon$, $\epsilon \sim \mathcal{N}(0, \sigma)$  and observe reward  $r$ and next state $s'$.
		\STATE Store transition tuple $\left\langle s, a, r, s'\right\rangle$ in $\mathcal{D}$.
		\STATE Sample a mini-batch of transitions $\left\{\left\langle s, a, r, s'\right\rangle\right\}$ from $\mathcal{D}$.
		\STATE Determine $\boldsymbol{\mathcal{M}}_K = \left\{ a_i \right\}_{i=1}^K, a_i \sim \mathcal{N}\left(\mu\left(s'; {\boldsymbol{\phi}_2^-}\right), \bar{\sigma}\right)$ and define $a_K^* = \arg \max_{a \in \boldsymbol{\mathcal{M}}_K} Q\left(s', a;\boldsymbol{\theta}_1^-\right)$.
		\STATE Update $\boldsymbol{\theta}_{i} \leftarrow \operatorname{argmin}_{\boldsymbol{\theta}_{i}} N^{-1} \sum{\left[r + \gamma \min\left\{Q\left(s', a_K^*;\boldsymbol{\theta}_2^-\right), Q\left(s', \mu\left(s'; {\boldsymbol{\phi}_1^-}\right);\boldsymbol{\theta}_1^-\right)\right\}-Q\left(s, a; \boldsymbol{\theta}_i\right)\right]}^{2}$.
		\IF{$t$ mod $d$}
		\STATE Update $\boldsymbol{\phi}_i$ by the deterministic policy gradient: $\nabla_{\boldsymbol{\phi}_i} J(\boldsymbol{\phi}_i)=\left.\frac{1}{N} \sum \nabla_{a} Q_{\theta_{i}}(s, a)\right|_{a=\mu\left(s; {\boldsymbol{\phi}_i}\right)} \nabla_{\boldsymbol{\phi}_i} \mu\left(s; {\boldsymbol{\phi}_i}\right)$.
		\STATE Update target networks: $\boldsymbol{\theta}_{i}^{-} \leftarrow \tau \boldsymbol{\theta}_{i}+(1-\tau) \boldsymbol{\theta}_{i}^{-}$, $\boldsymbol{\phi}_i^{-} \leftarrow \tau \boldsymbol{\phi}_i+(1-\tau) \boldsymbol{\phi}_i^{-}$.
		\ENDIF
		\ENDFOR
	\end{algorithmic}
\end{algorithm*}

\subsection{Action Candidate Based Clipped Double Estimator}
Double estimator is essentially an underestimated estimator, leading to the underestimation bias. The clipping operation in the clipped double estimator further exacerbates the underestimation problem. Therefore, although the clipped double estimator can effectively avoid the positive bias, it generates the large negative bias.

In order to reduce the negative bias of the clipped double estimator, we propose an action candidate based clipped double estimator denoted as $\hat{\mu}_{AC}^*$.
Notably, the double estimator chooses the index $a^*$  only from the estimator set $\hat{\boldsymbol{\mu}}^A$ and ignores the other estimator set $\hat{\boldsymbol{\mu}}^B$. Thus, it may choose the index $a^*$ associated with the low value in $\hat{\boldsymbol{\mu}}^B$ and generate the small estimation $\hat{\mu}_{a^*}^B$, leading to the large negative bias.
Different from the double estimator, instead of selecting the index $a^*$ from $\hat{\boldsymbol{\mu}}^A$ among all indexes, we just choose it from an index subset called candidates. 
The set of candidates, denoted as $\boldsymbol{\mathcal{M}}_K$, is defined as the index subset corresponding to the largest $K$ values in $\hat{\boldsymbol{\mu}}^B$, that is:
\begin{normalsize}
	\begin{equation} \label{policy}
	\begin{split}
	\boldsymbol{\mathcal{M}}_K = \left\{i|\hat{\mu}_i^B \in {\rm\ top} \ K \ {\rm values \ in \ \hat{\boldsymbol{\mu}}^B} \right\}.
	\end{split}
	\end{equation}
\end{normalsize}
The variable $a_K^*$ is then selected as the index to maximize $\hat{\boldsymbol{\mu}}^A$ among the index subset $\boldsymbol{\mathcal{M}}_K$: $\hat{\mu}_{a_K^*}^A=\max_{i \in \boldsymbol{\mathcal{M}}_K}\hat{\mu}_i^A$. 
If there are multiple indexes owning the maximum value, we randomly pick one. 
Finally, by clipping, we estimate the maximum expected value as below:
\begin{normalsize}
	\begin{equation} \label{policy}
	\begin{split}
	\mu^*={\rm max}_{i}\mu_i={\rm max}_{i}\mathbb{E}\left[\hat{\mu}_i^B\right] \approx \min\left\{ \hat{\mu}_{a_K^*}^B, \hat{\mu}_{SE}^*\right\}.
	\end{split}
	\end{equation}
\end{normalsize}

Consequently, we theoretically analyze the estimation bias of action candidate based clipped double estimator.
\begin{theorem}
	As the number $K$ decreases, the underestimation decays monotonically, that is $\mathbb{E}\left[\min\left\{\hat{\mu}_{a_{K}^*}^B, \hat{\mu}^*_{SE}\right\}\right] \geq \mathbb{E}\left[\min\left\{\hat{\mu}_{a_{K+1}^*}^B, \hat{\mu}^*_{SE}\right\}\right]$, $1\leq K < N$, where the inequality is strict if and only if $P\left( \hat{\mu}_{SE}^* > \hat{\mu}_{a_{K}^*}^B> \hat{\mu}_{a_{K+1}^*}^B\right)>0$ or $P\left(\hat{\mu}_{a_{K}^*}^B \geq \hat{\mu}_{SE}^*>\hat{\mu}_{a_{K+1}^*}^B\right)>0$. Moreover, $\forall{K}: 1\leq K \leq N$, $\mathbb{E}\left[\min\left\{\hat{\mu}_{a_K^*}^B, \hat{\mu}_{SE}^*\right\}\right] \geq \mathbb{E}\left[\hat{\mu}_{CDE}^*\right]$.
\end{theorem}
Notably, from the last inequality in Theorem 1, one can see that our estimator can effectively reduce the large underestimation bias in clipped double estimator. Moreover, since the existed inequality $\mathbb{E}\left[\hat{\mu}_{SE}^*\right]\geq\mathbb{E}\left[\min\left\{ \hat{\mu}_{a_K^*}^B,\hat{\mu}_{SE}^*\right\}\right]\geq\mathbb{E}\left[\hat{\mu}_{CDE}^*\right]$,
it essentially implies that the choice of $K$ controls the trade-off between the overestimation bias in single estimator and the underestimation bias in clipped double estimator. For the proof please refer to Appendix A.

The upper bound of 
$\mathbb{E}\left[\hat{\mu}_{SE}^{*}\right]$ \cite{van2013estimating} is:
\begin{small}
	\begin{equation} \label{policy}
	\begin{split}
	\mathbb{E}\left[\hat{\mu}_{SE}^{*}\right] = \mathbb{E}\left[ {\rm max}_i\hat{\mu}_i \right]  \leq \mu^* + \sqrt{\frac{N-1}{N} \sum_{i}^{N} \operatorname{Var}\left[\hat{\mu}_{i}\right]}.
	\end{split}
	\end{equation}
\end{small}
Since $\mathbb{E}\left[\hat{\mu}_{a_{K}^*}^B\right]$ decreases monotonically as the number $K$ increases (see Property 1 in Appendix A), $\mathbb{E}\left[\hat{\mu}_{a_{1}^*}^B\right]$ is maximum.
Due to the candidate subset $\mathcal{M}_1$ only contains one candidate corresponding to the largest value in $\hat{\boldsymbol{\mu}}^B$, we can obtain $\mathbb{E}\left[\hat{\mu}_{a_{1}^*}^B\right] = \mathbb{E}\left[\max_i\hat{\mu}_i^B\right]$.
Similar to the upper bound in $\mathbb{E}\left[\hat{\mu}_{SE}^{*}\right]$, we can see that $\mathbb{E}\left[\max_i\hat{\mu}_i^B\right] \leq \mu^* + \sqrt{\frac{N-1}{N} \sum_{i}^{N} \operatorname{Var}\left[\hat{\mu}_{i}^B\right]}$.
Since $\hat{\mu}_{i}^B$ is just estimated via $S_i^B$ containing half of samples rather than $S_i$, $\operatorname{Var}\left[\hat{\mu}_{i}\right] \leq \operatorname{Var}\left[\hat{\mu}_{i}^B\right]$ and thus $\mu^* + \sqrt{\frac{N-1}{N} \sum_{i}^{N} \operatorname{Var}\left[\hat{\mu}_{i}\right]} \leq  \mu^* + \sqrt{\frac{N-1}{N} \sum_{i}^{N} \operatorname{Var}\left[\hat{\mu}_{i}^B\right]}$.
So, such larger upper bound may cause the maximum value $\mathbb{E}\left[\hat{\mu}_{a_{1}^*}^B\right]$ to exceed the $\mathbb{E}\left[\hat{\mu}_{SE}^{*}\right]$.
Meanwhile, based on the monotonicity in Property 1, it further implies that when number $K$ is too small, the upper of $\mathbb{E}\left[\hat{\mu}_{a_{K}^*}^B\right]$ tends to be larger than the one of $\mathbb{E}\left[\hat{\mu}_{SE}^{*}\right]$, which may cause larger overestimation bias.   
Therefore, the clipping operation guarantees that no matter how small the number of the selected candidates is, the overestimation bias of our estimator is no more than that of the single estimator.

\section{Action Candidate Based Clipped Double Estimator for Double Q-learning and TD3}
In this section, we apply our proposed action candidate based clipped double estimator into Double Q-learning and TD3. 
For the discrete action task, we first propose the action candidate based clipped Double Q-learning in the tabular setting, and then generalize it to the deep case with the deep neural network, that is action candidate based clipped Double DQN.
For the continuous action task, we combine our estimator with TD3 and form action candidate based TD3.

\subsection{Action Candidate Based Clipped Double Q-learning}
\subsubsection{Tabular Version}
In tabular setting, action candidate based Double Q-learning stores the Q-functions $Q^A$ and $Q^B$, and learns them from two separate subsets of the online collected experience. Each Q-function is updated with a value from the other Q-function for the next state. 
Specifically, in order to update $Q^A$, we first determine the action candidates:
\begin{small}
	\begin{equation} \label{policy}
	\begin{split}
	\boldsymbol{\boldsymbol{\mathcal{M}}}_K = \left\{i | Q^B(s', a_i) \in {\rm  top\ } K {\rm \ values \ in} \ Q^B(s', \cdot) \right\}.
	\end{split}
	\end{equation}
\end{small}
According to the action value function $Q^A$, the action $a_K^*$ is the maximal valued action in the state $s'$ among $\boldsymbol{\mathcal{M}}_K$.
Then, we update $Q^A$ via the target value as below:
\begin{small}
	\begin{equation} \label{eq1}
	\begin{split}
	y^{AC\_CDQ}=r + \gamma \min\left\{Q^B(s', a_K^*), {\rm max}_{a} Q^A(s', a)\right\}.
	\end{split}
	\end{equation}
\end{small}
During the training process, 
the explored action is calculated with $\epsilon$-greedy exploration strategy based on the action values $Q^A$ and $Q^B$,
More details are shown in Algorithm~\ref{alg:algorithm1}. Note that in the tabular version, the number of action candidates balances the overestimation in Q-learning and the underestimation in clipped Double Q-learning.
\subsubsection{Deep Version}
For the task with the high-dimensional sensory input, we further propose the deep version of action candidate based clipped Double Q-learning, named action candidate based clipped Double DQN.
In our framework, we maintain two deep Q-networks and an experience buffer. In each update, we independently sample a batch of experience samples to train each Q-network   with the  target value in Eq.~\ref{eq1}.
Similar to the tabular version, the number of action candidates can also  balance the overestimation in DQN and the underestimation in clipped Double DQN. 
In addition, we verify that the action candidate based clipped Double Q-learning can converge to the optimal policy in the finite MDP setting. The proof can be seen in Appendix B.
\subsection{Action Candidate Based TD3}
As shown in Algorithm 2, the algorithm framework for the continuous action task follows the design in TD3.
To approximate the optimal action values, we construct two Q-networks $Q\left(s, a;\boldsymbol{\theta}_1\right)$ and $Q\left(s, a;\boldsymbol{\theta}_2\right)$ and two target Q-networks $Q\left(s, a;\boldsymbol{\theta}_1^-\right)$ and $Q\left(s, a;\boldsymbol{\theta}_2^-\right)$.
In addition, two deterministic policy networks $\mu\left(s; {\boldsymbol{\phi}_1}\right)$ and $\mu\left(s; {\boldsymbol{\phi}_2}\right)$, and two target networks $\mu\left(s; {\boldsymbol{\phi}_1^-}\right)$ and $\mu\left(s; {\boldsymbol{\phi}_2^-}\right)$ are exploited to represent the optimal decisions corresponding to  $Q\left(s, a;\boldsymbol{\theta}_1\right)$, $Q\left(s, a;\boldsymbol{\theta}_2\right)$,  $Q\left(s, a;\boldsymbol{\theta}_1^-\right)$ and $Q\left(s, a;\boldsymbol{\theta}_2^-\right)$.

Due to the continuity of the actions, it is impossible to precisely determine the top $K$ action candidates $\boldsymbol{\mathcal{M}}_K$ like in the discrete action case. We first exploit our deterministic policy network $\mu\left(s'; {\boldsymbol{\phi}_2^-}\right)$ to approximate the global optimal action $a^* = {\rm arg} \max_a Q\left(s', a;\boldsymbol{\theta}_2^-\right)$. Based on the estimated global optimal action $a^*$, we randomly select $K$ actions $\boldsymbol{\mathcal{M}}_{K}$ in the $\delta$-neighborhood of $a^*$ as the action candidates. 
Specifically, we draw $K$ samples from a Gaussian distribution $\mathcal{N}\left(\mu\left(s'; {\boldsymbol{\phi}_2^-}\right), \bar{\sigma}\right)$:
\begin{small}
	\begin{equation} \label{policy}
	\begin{split}
	\boldsymbol{\mathcal{M}}_K = \left\{ a_i | a_i \sim \mathcal{N}\left(\mu\left(s'; {\boldsymbol{\phi}_2^-}\right), \bar{\sigma}\right), i = 1, 2, \ldots ,K  \right\},
	\end{split}
	\end{equation}
\end{small}
where the hyper-parameter $\bar{\sigma}$ is the standard deviation. Both Q-networks  are updated via the following target value:
\begin{small}
	\begin{equation} \label{policy}
	\begin{split}
	y^{AC\_TD3} = r + \gamma \min\left\{Q\left(s', a_K^*;\boldsymbol{\theta}_2^-\right), Q\left(s', \mu\left(s'; {\boldsymbol{\phi}_1^-}\right);\boldsymbol{\theta}_1^-\right)\right\},
	\end{split}
	\end{equation}
\end{small}
where $a_K^* = \arg\max_{a \in \boldsymbol{\mathcal{M}}_K} Q\left(s', a;\boldsymbol{\theta}_1^-\right)$. 
The parameters of two policy networks are updated along the direction that can improve their corresponding Q-networks.
For more details please refer to Algorithm 2.

\begin{figure}[t]
	\centering  
	\subfigure[Number of impressions ($\times10^4$)]{
		\label{Fig.sub.1}
		\includegraphics[width=\columnwidth]{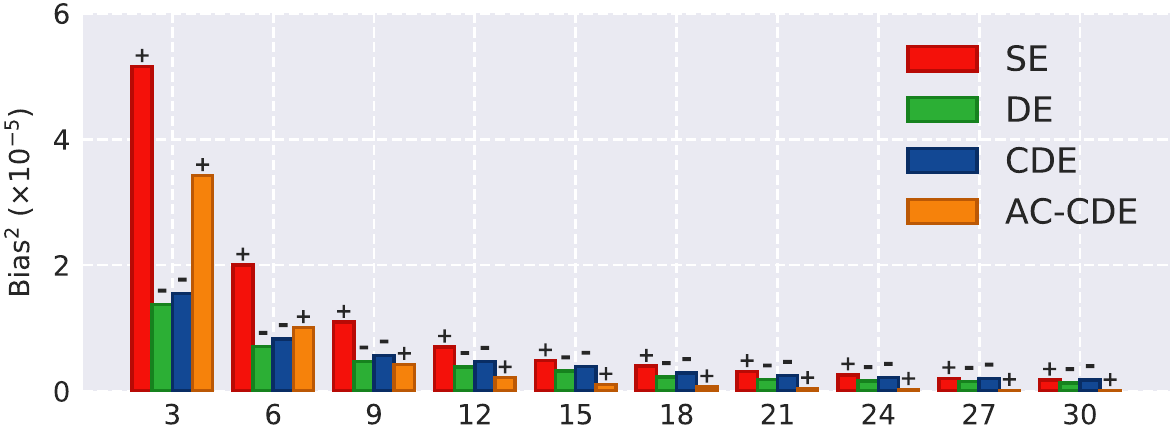}}
	\subfigure[Number of ads]{
		\label{Fig.sub.2}
		\includegraphics[width=\columnwidth]{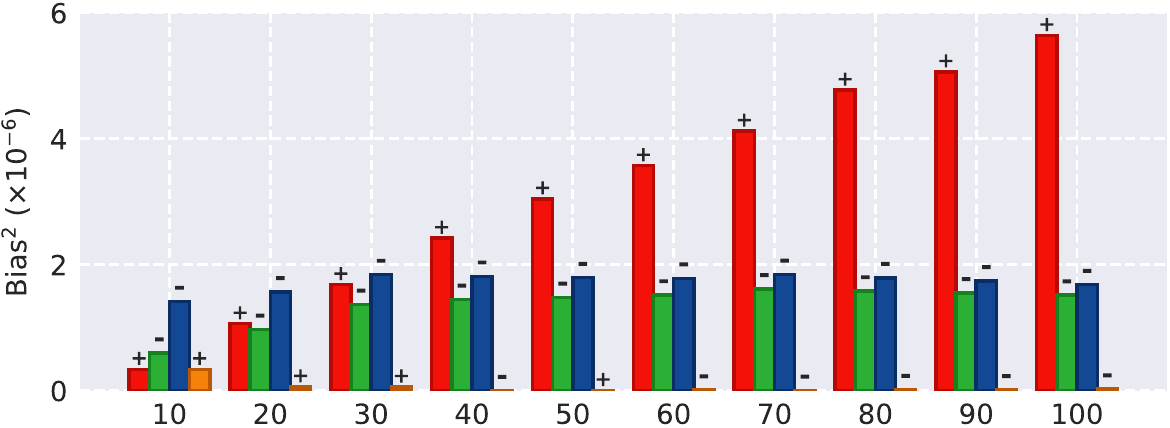}}
	\subfigure[Max probability in range ($\times 10^{-2}$)]{
		\label{Fig.sub.3}
		\includegraphics[width=\columnwidth]{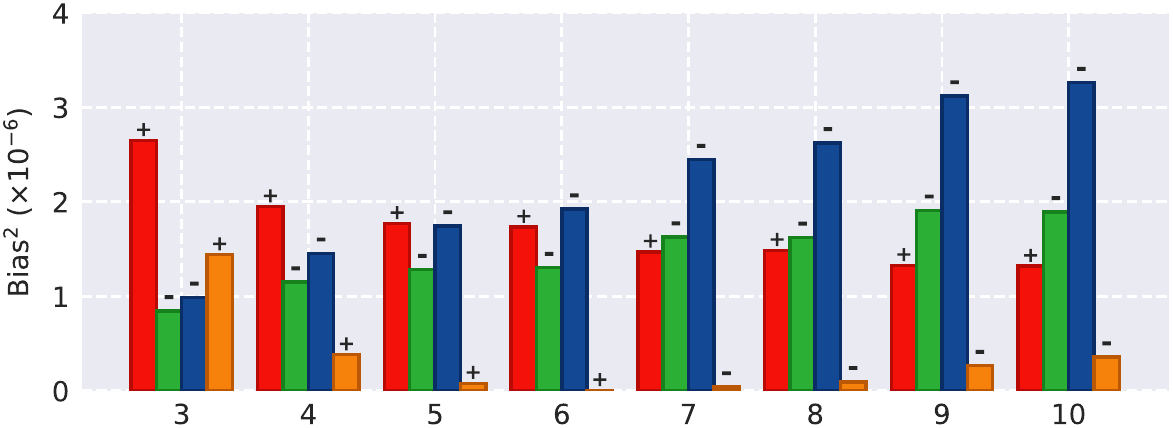}}
	\caption{Comparison on the multi-armed bandits for internet ads in three cases: (a) Varying the number of impressions; (b) Varying the number of ads; (c) Varying the max probability. 
		The symbol on the bar represents the sign of the bias.
		The results are averaged over $2,000$ experiments. 
		We use $15\%$ of actions as the action candidates.		
	}
	\label{fig:games}
\end{figure}

\begin{figure*}[ht]
	\centering
	\includegraphics[width=\textwidth]{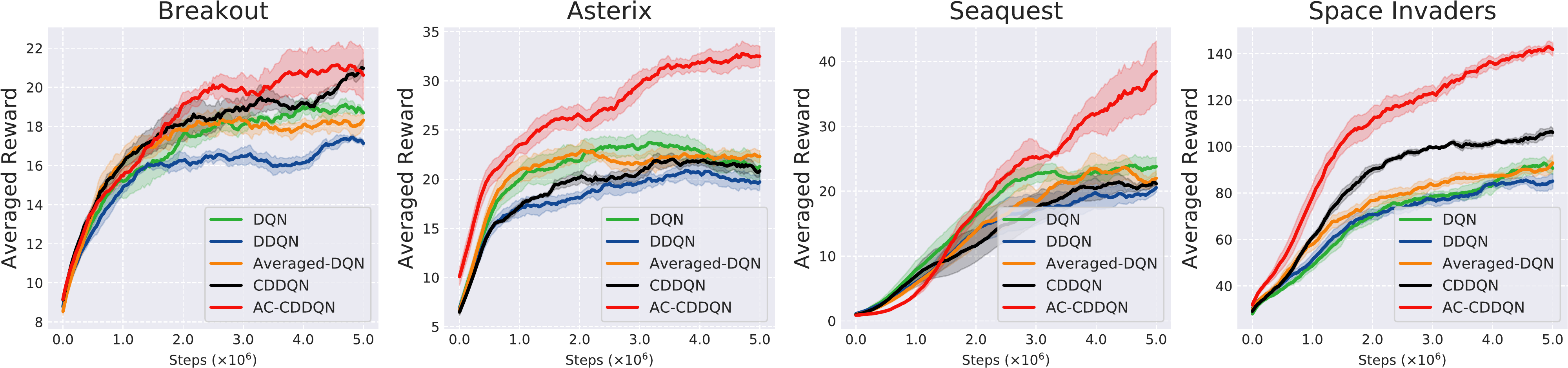}
	\caption{Learning curves on the four MinAtar benchmark environments. The results are averaged over five independent learning trials and the shaded area represents half a standard deviation. 
	}	
	\label{fig:minatari}
\end{figure*}

\section{Experiments}
\label{section:5}
In this section, we empirically evaluate our method on the discrete and continuous action tasks.

For the discrete action tasks, we conduct the following three experiments: 
\begin{itemize}
	\item For action candidate based clipped double estimator (AC-CDE), we compare them with single estimator \cite{hasselt2010double}, double estimator \cite{hasselt2010double} and clipped double estimator \cite{fujimoto2018addressing} on the multi-armed bandits problem.
	\item For action candidate based clipped Double Q-learning (AC-CDQ), we compare them with Q-learning \cite{watkins1992q}, Double Q-learning \cite{hasselt2010double} and clipped Double Q-learning \cite{fujimoto2018addressing} on grid world game.
	\item For action candidate based clipped Double DQN (AC-CDDQN), we compare them with DQN \cite{mnih2015human}, Double DQN \cite{van2016deep}, Averaged-DQN \cite{anschel2017averaged} and clipped Double DQN \cite{fujimoto2018addressing}  on several benchmark games in MinAtar \cite{young2019minatar}.
\end{itemize}

For the continuous action tasks, we compare our action candidate based TD3 (AC-TD3) with TD3 \cite{fujimoto2018addressing}, SAC \cite{haarnoja2018soft} and DDPG \cite{lillicrap2015continuous} on six MuJoCo \cite{todorov2012mujoco} based benchmark tasks implemented in OpenAI Gym \cite{dhariwal2017openai}.

\subsection{Discrete Action Tasks}
\subsubsection{Multi-Armed Bandits For Internet Ads}
In this experiment, we employ the framework of the multi-armed bandits to choose the best ad to show on the website among a set of $M$ possible ads, each one with an unknown fixed expected return per visitor.
For simplicity, we assume each ad has the same return per click, such that the best ad is the one with the maximum click rate. We model the click event per visitor in each ad $i$ as the Bernoulli event with mean $m_i$ and variance $(1-m_i)m_i$. In addition, all ads are assumed to have the same visitors, which means that given $N$ visitors, $N/M$ Bernoulli experiments will be executed to estimate the click rate of each ad. The default configuration in our experiment is $N=30,000$, $M=30$ and the mean click rates uniformly sampled from the interval $\left[0.02, 0.05\right]$. Based on this configuration, there are three settings: (1) We vary the number of visitors $N=\left\{30,000, 60,000,\ldots,270,000, 300,000\right\}$. (2) We vary the number of ads $M=\left\{10, 20,\ldots,90,100\right\}$. (3) We vary the upper limit of the sampling interval of mean click rate (the original is $0.05$) with values $\left\{0.03, 0.04,\ldots,0.09,0.1\right\}$.

To compare the absolute bias, we evaluate the single estimator, double estimator, clipped double estimator and AC-CDE with the square of bias ($bias^2$) in each setting.
As shown in Fig.~\ref{fig:games}, compared to other estimators, AC-CDE owns the lowest $bias^2$ in almost all experimental settings. 
It mainly benefits from the effective balance of our proposed estimator between the overestimation bias of single estimator and underestimation bias of clipped double estimator.
Moreover, AC-CDE has the lower $\textit{bias}^2$ than single estimator in all cases while in some cases it has the larger $bias^2$ than clipped double estimator such as the leftmost columns in Fig.~\ref{fig:games} (a) and (c). It's mainly due to that although AC-CDE can reduce the underestimation bias of clipped double estimator, too small number of action candidates may also in turn cause overestimation bias. 
Thus, the absolute value of such overestimation bias may be larger than the one of the underestimation bias in clipped double estimator. Despite this, AC-CDE can guarantee that the positive bias is no more than single estimator and the negative bias is also no more than clipped double estimator.

\subsubsection{Grid World} 
As a MDP task, in a $N\times N$ grid world, there are total $N^2$ states. The starting state $s_0$ is in the lower-left cell and the goal state is in the upper-right cell. Each state has four actions: east, west, south and north.  At any state, moving to an adjacent cell is deterministic, but a collision with the edge of the world will result in no movement. Taking an action at any state will receive a random reward which is set as below: if the next state is not the goal state, the random reward is $-6$ or $+4$ and if the agent arrives at the goal state, the random reward is $-30$ or $+40$. With the discount factor $\gamma$, the optimal value of the maximum value action in the starting state $s_0$ is $5\gamma^{2(N-1)}-\sum_{i=0}^{2N-3}\gamma^i$. 
We set $N$ to $5$ and $6$ to construct our grid world environments and compare the Q-learning, Double Q-learning, clipped Double Q-learning and AC-CDQ ($K=2,3$) on the mean reward per step and estimation error (see Fig.~\ref{fig:gridworld1}). 

From the top plots, one can see that AC-CDQ ($K=2,3$) can obtain the higher mean reward than other methods in both given environments.
We further plot the estimation error about the optimal state value $V^*(s_0)$ in bottom plots. 
Compared to Q-learning, Double Q-learning and clipped Double Q-learning, AC-CDQ ($K=2,3$) show the much lower estimation bias (more closer to the dash line), which means that it can better assess the action value and thus help generate more valid action decision.
Moreover, our AC-CDQ can significantly reduce the underestimation bias in clipped Double Q-learning.
Notably, as demonstrated in Theorem 1, 
the underestimation bias in the case of $K=2$ is smaller than that in the case of $K=3$.
And AC-CDQ can effectively balance the overestimation bias in Q-learning and the underestimation bias in clipped Double Q-learning.
%
%
%
%
\begin{figure}[h]
	\centering
	\includegraphics[width=1\columnwidth]{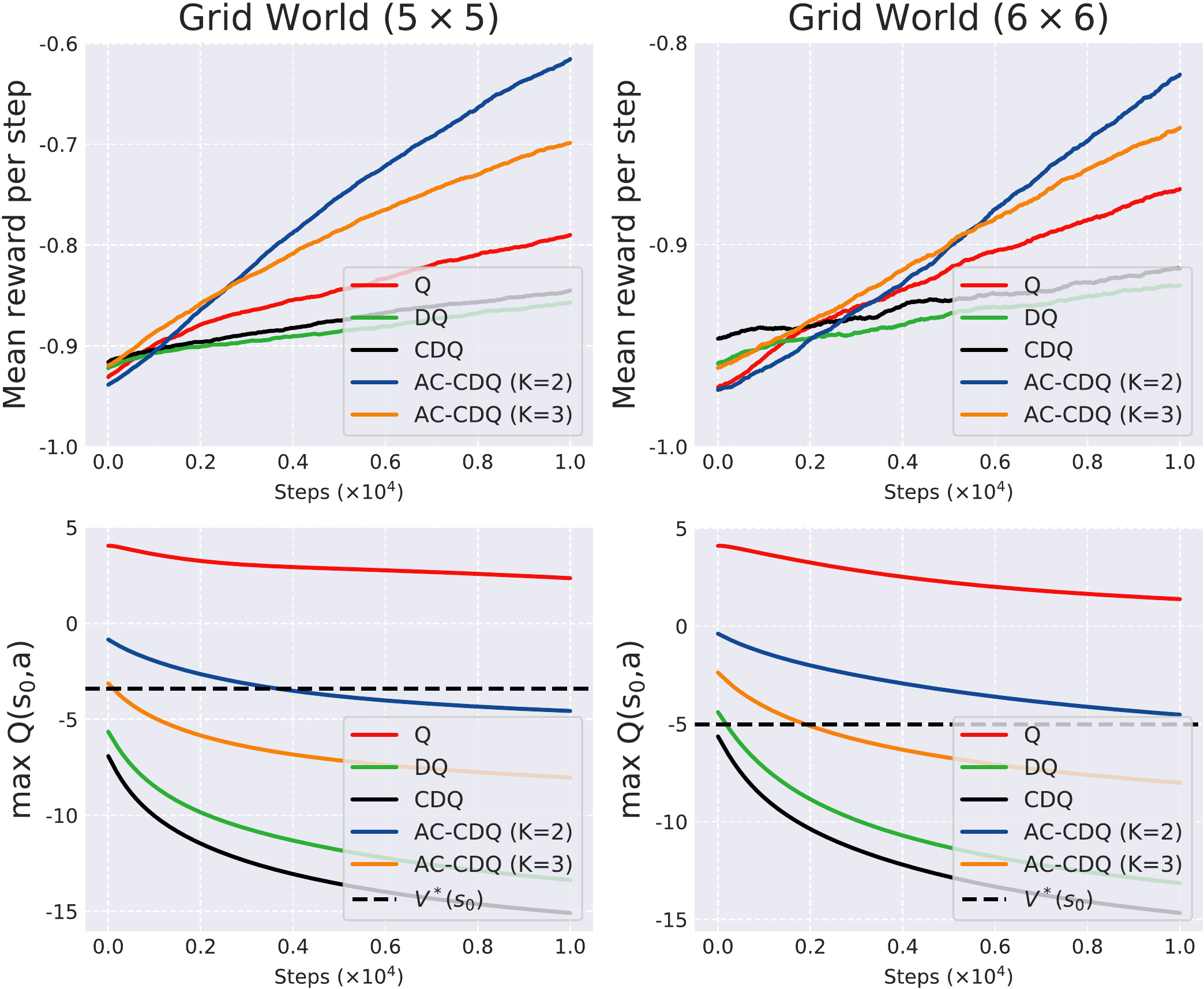}
	\caption{
		The top plots show the mean reward per step and the bottom plots show the estimated maximum action value from the state $s_0$ (the black dash line demotes the optimal state value $V^*(s_0)$). The results are averaged over 10000 experiments and each experiment contains 10000 steps. 
		We set the number of the action candidates to 2 and 3, respectively. 
	}	
	\label{fig:gridworld1}
\end{figure}
\subsubsection{MinAtar}
MinAtar is a game platform for testing the reinforcement learning algorithms, which uses a simplified state representation to model the game dynamics of Atari from ALE \cite{bellemare2013arcade}.
In this experiment, we compare the performance of DQN, Double DQN, Averaged-DQN, clipped Double DQN and AC-CDDQN on four released MinAtar games including Breakout, Asterix, Seaquest and Space Invaders.
We exploit the convolutional neural network as the function approximator and use the game image as the input to train the agent in an end-to-end manner. 
Following the settings in \cite{young2019minatar}, the hyper-parameters and settings of neural networks are set as follows: the batch size is 32; the replay memory size is 100,000; the update frequency is 1; the discounted factor is 0.99; the learning rate is 0.00025; the initial exploration is 1; the final exploration is 0.1; the replay start size is 5,000. The optimizer is RMSProp with the gradient momentum 0.95 and the minimum squared gradient 0.01. The experimental results are obtained after 5M frames.

Fig.~\ref{fig:minatari} represents the training curve about averaged reward of each algorithm. 
It shows that compared to DQN, Double DQN, Averaged-DQN and clipped Double DQN, AC-CDDQN can obtain better or comparable performance while they have the similar convergence speeds in all four games. 
Especially, for Asterix, Seaquest and Space Invaders, AC-CDDQN can achieve noticeably higher averaged rewards compared to the clipped Double DQN and obtain the gains of $36.3\%$, $74.4\%$ and $19.8\%$, respectively. 
Such significant gain mainly owes to that AC-CDDQN can effectively balance the overestimation bias in DQN and the underestimation bias in clipped Double DQN.
Moreover, in Fig.~\ref{fig:minatari_vary_k} we also test the averaged rewards of different numbers of action candidates $K=\left\{2,3,4\right\}$ for AC-CDDQN. The plots show that AC-CDDQN is consistent to obtain the robust and superior performance with different action candidate sizes.
\begin{figure}[t]
	\centering
	\includegraphics[width=1\columnwidth]{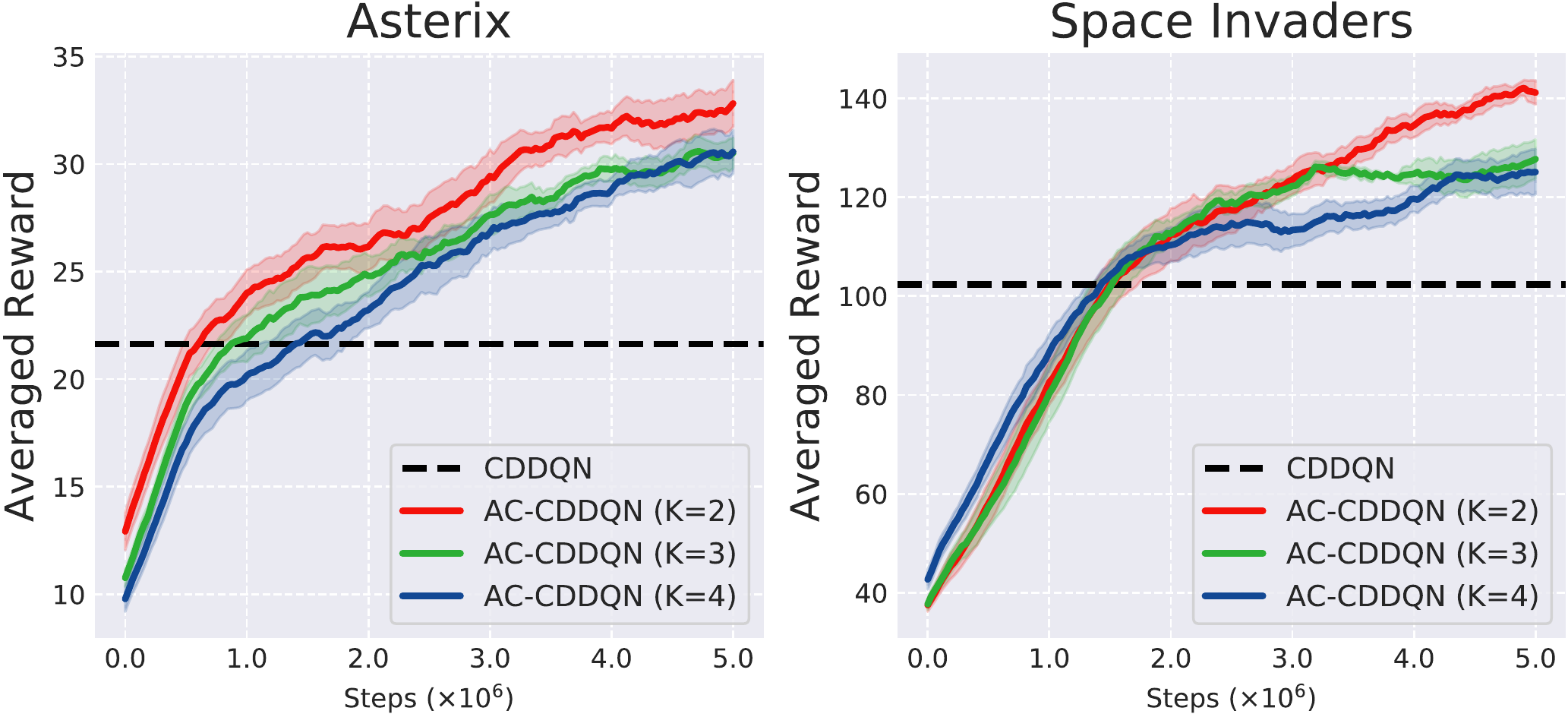}
	\caption{Learning curves on the several MinAtar benchmark environments for AC-CDDQN with different numbers of the action candidates. 
	}	
	\label{fig:minatari_vary_k}
\end{figure}

\subsection{Continuous Action Task}
\subsubsection{MuJoCo Tasks}
We verify our variant for continuous action, AC-TD3, on six MuJoCo continuous control tasks from OpenAI Gym including Ant-v2, Walker2D-v2, Swimmer-v2, Pusher-v2, Hopper-v2 and Reacher-v2.
We compare our method against the DDPG and two state of the art methods: TD3 and SAC.
In our method, we exploit the TD3 as our baseline and just modify it with our action candidate mechanism.	
The number of the action candidate is set to $32$.
We run all tasks with 1 million timesteps and the trained policies are evaluated every $5,000$ timesteps.

We list the training curves of Ant-v2 and Swimmer-v2 in the top row of Fig.~\ref{fig:mujoco} and more curves are listed in Appendix C.
The comprehensive comparison results are listed in Table~\ref{mujoco_table}.
From Table~\ref{mujoco_table}, one can see that DDPG performs poorly in most environments and TD3 and SAC can't handle some tasks such as Swimmer-v2 well.
In contrast, AC-TD3 consistently obtains the robust and competitive performance in all environments.
Particularly, AC-TD3 owns comparable learning speeds across all tasks and can achieve higher averaged reward than TD3 (our baseline) in most environments except for Hopper-v2. 
Such significant performance gain demonstrates that our proposed approximate action candidate method in the continuous action case is effective empirically.
Moreover, we also explain the performance advantage of our AC-TD3 over TD3 from the perspective of the bias (see the bottom row of the Fig.~\ref{fig:mujoco}).
The bottom plots show that in Ant-v2 and Swimmer-v2, AC-TD3 tends to have a lower estimation bias than TD3 about the expected return with regard to the initial state $s_0$, which potentially helps the agent assess the action at some state better and then generate the more reasonable policy.
\begin{figure}[t]
	\centering
	\includegraphics[width=1\columnwidth]{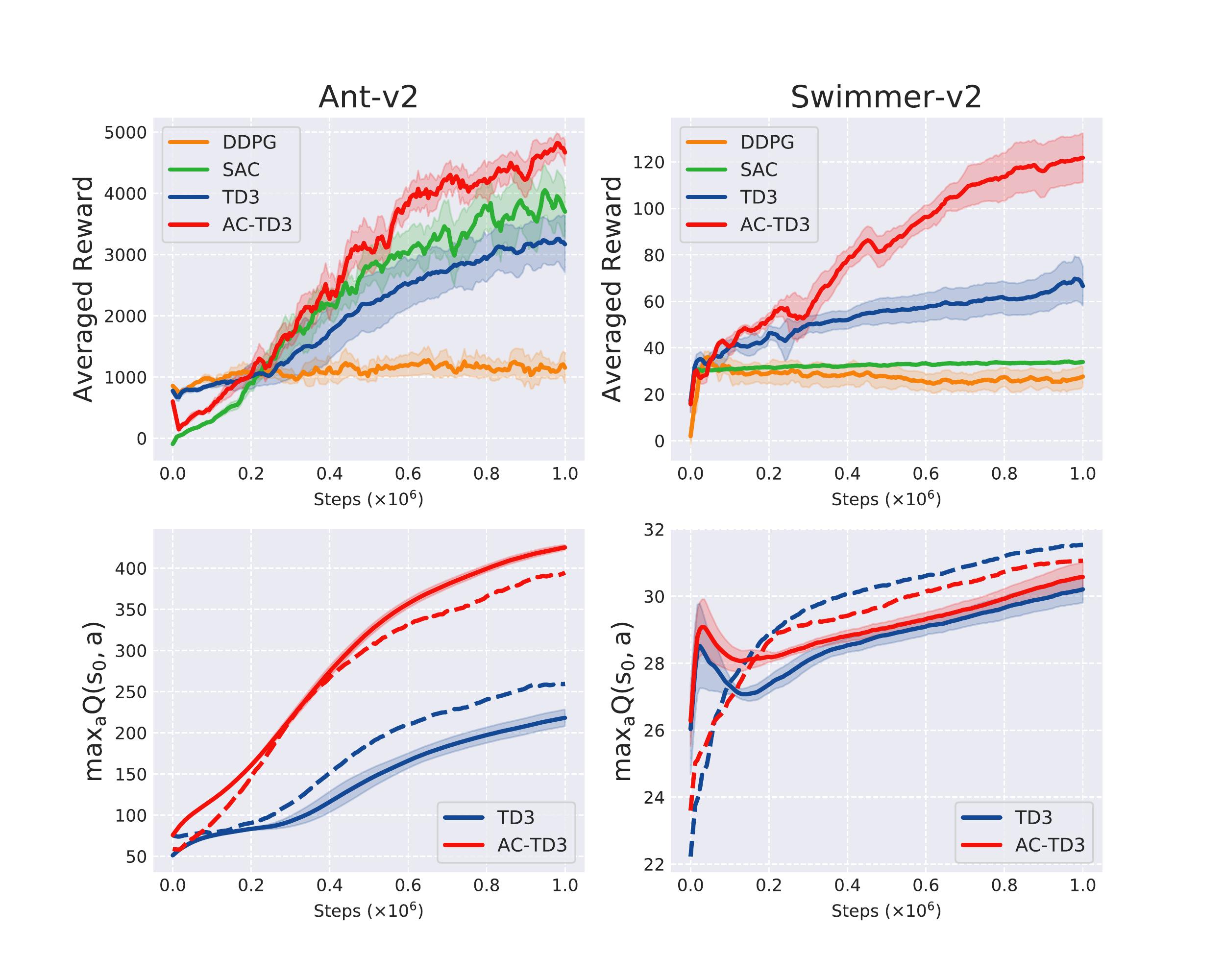}
	\caption{Top row: Learning curves for the OpenAI Gym continuous control tasks. The shaded region represents half a standard deviation of the average evaluation over 10 trials. Bottom plots: the estimation of the expected return with respect to the initial state $s_0$ of the game. The dash lines represent the real discounted return.
	}	
	\label{fig:mujoco}
\end{figure}

\begin{table}[htbp]
	\centering
	\resizebox{1\linewidth}{!}{
		\begin{tabular}{c|c|c|c|c}
			\toprule
			& AC-TD3  & TD3 & SAC & DDPG \\
			\midrule	
			Pusher   & \textbf{-22.7 $\pm$ 0.39}  & -31.8 & -76.7 & -38.4\\
			Reacher & \textbf{-3.5 $\pm$ 0.06}  & \textbf{-3.6} & -12.9 & -8.9 \\
			Walker2d &  \textbf{3800.3 $\pm$ 130.95}  & 3530.4 & 1863.8 & 1849.9 \\
			Hopper  & 2827.2 $\pm$ 83.2 & {2974.8} & \textbf{3111.1} & 2611.4 \\
			Swimmer &  \textbf{116.2 $\pm$ 3.63} & 63.2 & 33.6 & 30.2 \\
			Ant &  \textbf{4391 $\pm$ 205.6} & 3044.6 & 3646.5 & 1198.64 \\
			\bottomrule
	\end{tabular}}
	\caption{Averaged rewards over last 30\% episodes during training process.}
	\label{mujoco_table}
\end{table}

\section{Conclusion}
\label{section:6}
In this paper, we proposed an action candidate based clipped double estimator to approximate the maximum expected value. Furthermore, we applied this estimator to form the action candidate based clipped Double Q-learning. Theoretically, the underestimation bias in clipped Double Q-learning decays monotonically as the number of action candidates decreases. The number of the action candidates can also control the trade-off between the overestimation and underestimation. 
Finally, we also extend our clipped Double Q-learning to the deep version and the continuous action tasks.
Experimental results demonstrate that our proposed methods yield competitive performance. 

\section{Acknowledgments}
This work was supported by the National Science Fund of China (Grant Nos. U1713208, 61876084), Program for Changjiang Scholars.

\bibliographystyle{aaai21.sty}
\bibliography{aaai}

\clearpage


\onecolumn

\section{Supplementary Material for ``Action Candidate Based Clipped Double Q-learning for Discrete and Continuous Action Tasks"}
\section{A. Proof of Theorems on Proposed Estimator}
\begin{property}
	Let $a_K^*$ be the index that maximizes $\hat{\boldsymbol{\mu}}^A$ among $\boldsymbol{\mathcal{M}}_K$: $\hat{\mu}_{a_K^*}^A=\max_{i \in \boldsymbol{\mathcal{M}}_K}\hat{\mu}_i^A$.
	Then, as the number $K$ decreases, the underestimation decays monotonically, that is $\mathbb{E}\left[\hat{\mu}_{a_{K}^*}^B\right] \geq \mathbb{E}\left[\hat{\mu}_{a_{K+1}^*}^B\right]$, $1\leq K < N$. 
\end{property}

\begin{proof}
	Suppose that $\boldsymbol{\mathcal{M}}_K=\{a_{(1)},\ldots,a_{(K)}\}$ for $\hat{\mu}_{a_{K}^*}^B$, where $a_{(i)}$ denotes the index corresponding to the $i$-th largest value in $\hat{\mu}^B$ and $a_K^*=\arg\max_{j \in \boldsymbol{\mathcal{M}}_K}\hat{\mu}_j^A$. Then, $\boldsymbol{\mathcal{M}}_{K+1} = \boldsymbol{\mathcal{M}}_K \cup \{a_{(K+1)}\}$ for $\hat{\mu}_{a_{K+1}^*}^B$. 
	If $\hat{\mu}_{a_{(K+1)}}^A > \hat{\mu}_{a_K^*}^A$, then $a_{K+1}^*=a_{(K+1)}$. 
	Due to $a_{(K+1)} \notin \boldsymbol{\mathcal{M}}_{K}$ and $a_K^* \in \boldsymbol{\mathcal{M}}_{K}$, $\hat{\mu}_{a_{(K+1)}}^B=\hat{\mu}_{a_{K+1}^*}^B < \hat{\mu}_{a_{K}^*}^B$. Similarly, if $\hat{\mu}_{a_{(K+1)}}^A < \hat{\mu}_{a_K^*}^A$, then $a_{K+1}^*$ is equal to $a_{K}^*$. Thus, $\hat{\mu}_{a_{K+1}^*}^B = \hat{\mu}_{a_{K}^*}^B$.
	Finally, if $\hat{\mu}_{a_{(K+1)}}^A = \hat{\mu}_{a_K^*}^A$, $a^*_{K+1}$ is either equal to $a_K^*$ or equal to $a_{(K+1)}$. For the former,  the estimation value under $K+1$ remain unchanged, that is $\hat{\mu}_{a_{K+1}^*}^B = \hat{\mu}_{a_{K}^*}^B$. For the latter, $\hat{\mu}^B_{a^*_{K+1}} = \hat{\mu}^B_{a_{(K+1)}} <= \hat{\mu}^B_{a^*_K}$where the equal sign is established only when there are multiple $K$-th largest values and $\hat{\mu}^B_{a_{(K)}} = \hat{\mu}^B_{a^*_K}$.
	Therefore, we can obtain $\hat{\mu}_{a_{K+1}^*}^B \leq \hat{\mu}_{a_{K}^*}^B$ and $\mathbb{E}\left[\hat{\mu}_{a_{K+1}^*}^B\right] \leq \mathbb{E}\left[\hat{\mu}_{a_{K}^*}^B\right]$.
	The inequality is strict if and only if $P\left(\hat{\mu}_{a_{K+1}^*}^B<\hat{\mu}_{a_{K}^*}^B\right)>0$.
\end{proof}

\begin{theorem}
	Assume that the values in $\hat{\boldsymbol{\mu}}^A$ and $\hat{\boldsymbol{\mu}}^B$ are different.
	Then, as the number $K$ decreases, the underestimation decays monotonically, that is $\mathbb{E}\left[\min\left\{\hat{\mu}_{a_{K}^*}^B, \hat{\mu}^*_{SE}\right\}\right] \geq \mathbb{E}\left[\min\left\{\hat{\mu}_{a_{K+1}^*}^B, \hat{\mu}^*_{SE}\right\}\right]$, $1\leq K < N$, where the inequality is strict if and only if $P\left( \hat{\mu}_{SE}^* > \hat{\mu}_{a_{K}^*}^B> \hat{\mu}_{a_{K+1}^*}^B\right)>0$ or $P\left(\hat{\mu}_{a_{K}^*}^B \geq \hat{\mu}_{SE}^*>\hat{\mu}_{a_{K+1}^*}^B\right)>0$. Moreover, $\forall{K}: 1\leq K \leq N$, $\mathbb{E}\left[\min\left\{\hat{\mu}_{a_K^*}^B, \hat{\mu}_{SE}^*\right\}\right] \geq \mathbb{E}\left[\hat{\mu}_{CDE}^*\right]$.
\end{theorem}

\begin{proof}
	For simplicity, we set $G(K) = \min\left\{\hat{\mu}_{a_{K}^*}^B, \hat{\mu}^*_{SE}\right\}-\min\left\{\hat{\mu}_{a_{K+1}^*}^B, \hat{\mu}^*_{SE}\right\}$.
	First, we have
	\begin{normalsize}
		\begin{equation} \label{eq1}
		\begin{split}
		\mathbb{E}\left[G\left(K\right)\right]&=P\left(\hat{\mu}_{a_{K}^*}^B <  \hat{\mu}_{SE}^*\right)\mathbb{E}\left[G\left(K\right)\mid\hat{\mu}_{a_{K}^*}^B <  \hat{\mu}_{SE}^*\right]
		+P\left(\hat{\mu}_{a_{K}^*}^B \geq  \hat{\mu}_{SE}^*\right)\mathbb{E}\left[G\left(K\right)\mid\hat{\mu}_{a_{K}^*}^B \geq  \hat{\mu}_{SE}^*\right].
		\end{split}
		\end{equation}
	\end{normalsize}
	Then, from Property 1, we have $\hat{\mu}_{a_{K+1}^*}^B \leq \hat{\mu}_{a_{K}^*}^B$, $1\leq K < N$. Hence, the expected value of $G(K)$ under the condition $\hat{\mu}_{a_{K}^*}^B <  \hat{\mu}_{SE}^*$ can be obtained as below:
	\begin{normalsize}
		\begin{equation} \label{policy}
		\begin{split}
		&\mathbb{E}\left[G\left(K\right)\mid\hat{\mu}_{a_{K}^*}^B <  \hat{\mu}_{SE}^*\right]=\mathbb{E}\left[\hat{\mu}_{a_{K}^*}^B-\hat{\mu}_{a_{K+1}^*}^B\mid\hat{\mu}_{a_{K}^*}^B <  \hat{\mu}_{SE}^*\right]\\
		=&P\left(\hat{\mu}_{a_{K}^*}^B > \hat{\mu}_{a_{K+1}^*}^B\mid\hat{\mu}_{a_{K}^*}^B < \hat{\mu}_{SE}^*\right)\underbrace{\mathbb{E}\left[\hat{\mu}_{a_{K}^*}^B-\hat{\mu}_{a_{K+1}^*}^B \mid \hat{\mu}_{a_{K+1}^*}^B <\hat{\mu}_{a_{K}^*}^B < \hat{\mu}_{SE}^*\right]}_{>0} \\
		&+P\left(\hat{\mu}_{a_{K}^*}^B = \hat{\mu}_{a_{K+1}^*}^B\mid\hat{\mu}_{a_{K}^*}^B < \hat{\mu}_{SE}^*\right)\underbrace{\mathbb{E}\left[\hat{\mu}_{a_{K}^*}^B-\hat{\mu}_{a_{K+1}^*}^B \mid \hat{\mu}_{a_{K+1}^*}^B =\hat{\mu}_{a_{K}^*}^B < \hat{\mu}_{SE}^*\right]}_{=0}\\
		=&P\left(\hat{\mu}_{a_{K}^*}^B > \hat{\mu}_{a_{K+1}^*}^B\mid\hat{\mu}_{a_{K}^*}^B < \hat{\mu}_{SE}^*\right)\underbrace{\mathbb{E}\left[\hat{\mu}_{a_{K}^*}^B-\hat{\mu}_{a_{K+1}^*}^B \mid \hat{\mu}_{a_{K+1}^*}^B <\hat{\mu}_{a_{K}^*}^B < \hat{\mu}_{SE}^*\right]}_{>0}.
		\end{split}
		\end{equation}
	\end{normalsize}
	Thus, the first item in Eq.~\ref{eq1} can be rewritten as below:
	\begin{normalsize}
		\begin{equation} \label{policy}
		\begin{split}
		&P\left(\hat{\mu}_{a_{K}^*}^B <  \hat{\mu}_{SE}^*\right)\mathbb{E}\left[G\left(K\right)\mid\hat{\mu}_{a_{K}^*}^B <  \hat{\mu}_{SE}^*\right]\\
		=&P\left(\hat{\mu}_{a_{K}^*}^B <  \hat{\mu}_{SE}^*\right)P\left(\hat{\mu}_{a_{K}^*}^B > \hat{\mu}_{a_{K+1}^*}^B\mid\hat{\mu}_{a_{K}^*}^B < \hat{\mu}_{SE}^*\right){\mathbb{E}\left[\hat{\mu}_{a_{K}^*}^B-\hat{\mu}_{a_{K+1}^*}^B \mid \hat{\mu}_{a_{K+1}^*}^B <\hat{\mu}_{a_{K}^*}^B < \hat{\mu}_{SE}^*\right]}\\
		=& P\left(\hat{\mu}_{a_{K+1}^*}^B <\hat{\mu}_{a_{K}^*}^B < \hat{\mu}_{SE}^*\right){\mathbb{E}\left[\hat{\mu}_{a_{K}^*}^B-\hat{\mu}_{a_{K+1}^*}^B \mid \hat{\mu}_{a_{K+1}^*}^B <\hat{\mu}_{a_{K}^*}^B < \hat{\mu}_{SE}^*\right]}
		\end{split}
		\end{equation}
	\end{normalsize}
	
	Moreover, the expected value of $G(K)$ under the condition $\hat{\mu}_{a_{K}^*}^B \geq  \hat{\mu}_{SE}^*$ under the condition $\hat{\mu}_{a_{K}^*}^B \geq \hat{\mu}_{SE}^*$ can be obtained as below:
	\begin{normalsize}
		\begin{equation} \label{policy}
		\begin{split}
		&\mathbb{E}\left[G\left(K\right)\mid\hat{\mu}_{a_{K}^*}^B \geq  \hat{\mu}_{SE}^*\right]
		=\mathbb{E}\left[\hat{\mu}_{SE}^*-\min\left\{\hat{\mu}_{a_{K+1}^*}^B, \hat{\mu}^*_{SE}\right\}\mid\hat{\mu}_{a_{K}^*}^B \geq  \hat{\mu}_{SE}^*\right]\\
		=& P\left(\hat{\mu}_{a_{K+1}^*}^B<\hat{\mu}_{SE}^*\mid\hat{\mu}_{a_{K}^*}^B \geq  \hat{\mu}_{SE}^*\right)\underbrace{\mathbb{E}\left[\hat{\mu}_{SE}^*-\hat{\mu}_{a_{K+1}^*}^B\mid\hat{\mu}_{a_{K}^*}^B \geq  \hat{\mu}_{SE}^*>\hat{\mu}_{a_{K+1}^*}^B\right]}_{>0}\\
		&+P\left(\hat{\mu}_{a_{K+1}^*}^B\geq\hat{\mu}_{SE}^*\mid\hat{\mu}_{a_{K}^*}^B \geq  \hat{\mu}_{SE}^*\right)\underbrace{\mathbb{E}\left[\hat{\mu}_{SE}^*-\hat{\mu}_{SE}^*\mid\hat{\mu}_{a_{K}^*}^B \geq \hat{\mu}_{a_{K+1}^*}^B \geq  \hat{\mu}_{SE}^*\right]}_{=0}\\
		=& P\left(\hat{\mu}_{a_{K+1}^*}^B<\hat{\mu}_{SE}^*\mid\hat{\mu}_{a_{K}^*}^B \geq  \hat{\mu}_{SE}^*\right){\mathbb{E}\left[\hat{\mu}_{SE}^*-\hat{\mu}_{a_{K+1}^*}^B\mid\hat{\mu}_{a_{K}^*}^B \geq  \hat{\mu}_{SE}^*>\hat{\mu}_{a_{K+1}^*}^B\right]}.
		\end{split}
		\end{equation}
	\end{normalsize}
	Therefore, the second item in Eq.~\ref{eq1} can be rewritten as below:
	\begin{normalsize}
		\begin{equation} \label{policy}
		\begin{split}
		&P\left(\hat{\mu}_{a_{K}^*}^B \geq  \hat{\mu}_{SE}^*\right)\mathbb{E}\left[G\left(K\right)\mid\hat{\mu}_{a_{K}^*}^B \geq  \hat{\mu}_{SE}^*\right]\\
		=&P\left(\hat{\mu}_{a_{K}^*}^B \geq  \hat{\mu}_{SE}^*\right)P\left(\hat{\mu}_{a_{K+1}^*}^B<\hat{\mu}_{SE}^*\mid\hat{\mu}_{a_{K}^*}^B \geq  \hat{\mu}_{SE}^*\right){\mathbb{E}\left[\hat{\mu}_{SE}^*-\hat{\mu}_{a_{K+1}^*}^B\mid\hat{\mu}_{a_{K}^*}^B \geq  \hat{\mu}_{SE}^*>\hat{\mu}_{a_{K+1}^*}^B\right]}\\
		=&P\left(\hat{\mu}_{a_{K+1}^*}^B<\hat{\mu}_{SE}^*\leq \hat{\mu}_{a_{K}^*}^B\right){\mathbb{E}\left[\hat{\mu}_{SE}^*-\hat{\mu}_{a_{K+1}^*}^B\mid\hat{\mu}_{a_{K}^*}^B \geq  \hat{\mu}_{SE}^*>\hat{\mu}_{a_{K+1}^*}^B\right]}.
		\end{split}
		\end{equation}
	\end{normalsize}
	
	Finally, the expected value of $G(K)$ can be expressed as:
	\begin{normalsize}
		\begin{equation} \label{policy}
		\begin{split}
		\mathbb{E}\left[G\left(K\right)\right]
		=&P\left(\hat{\mu}_{a_{K+1}^*}^B<\hat{\mu}_{a_{K}^*}^B <  \hat{\mu}_{SE}^*\right)\underbrace{\mathbb{E}\left[\hat{\mu}_{a_{K}^*}^B-\hat{\mu}_{a_{K+1}^*}^B|\hat{\mu}_{a_{K+1}^*}^B<\hat{\mu}_{a_{K}^*}^B <  \hat{\mu}_{SE}^*\right]}_{>0}\\
		&+P\left(\hat{\mu}_{a_{K}^*}^B \geq  \hat{\mu}_{SE}^*>\hat{\mu}_{a_{K+1}^*}^B\right)\underbrace{\mathbb{E}\left[\hat{\mu}_{SE}^*-\hat{\mu}_{a_{K+1}^*}^B|\hat{\mu}_{a_{K}^*}^B \geq  \hat{\mu}_{SE}^*>\hat{\mu}_{a_{K+1}^*}^B\right]}_{>0}
		\geq  0,
		\end{split}
		\end{equation}
	\end{normalsize}
	where the inequality is strict if and only if $P\left( \hat{\mu}_{SE}^* > \hat{\mu}_{a_{K}^*}^B> \hat{\mu}_{a_{K+1}^*}^B\right)>0$ or $P\left(\hat{\mu}_{a_{K}^*}^B \geq \hat{\mu}_{SE}^*>\hat{\mu}_{a_{K+1}^*}^B\right)>0$.
	
	Further, due to the monotonicity of the expected value of $\min\left\{\hat{\mu}_{a_{K}^*}^B, \hat{\mu}^*_{SE}\right\}$ with regard to $K$ ($1\leq K \leq N$), we can know that the minimum value is at $K=N$, that is $\mathbb{E}\left[\min\left\{\hat{\mu}_{a_{K}^*}^B, \hat{\mu}^*_{SE}\right\}\right] \geq \mathbb{E}\left[\min\left\{\hat{\mu}_{a_{N}^*}^B, \hat{\mu}^*_{SE}\right\}\right]$. Since $\hat{\mu}_{a_{N}^*}^B$ means that we choose the index corresponding to the largest value in $\hat{\boldsymbol{\mu}}^A$ among the all indexs, which is equal to the double estimator, we can further have $\hat{\mu}_{a_{N}^*}^B=\hat{\mu}_{DE}^*$. Hence, $\mathbb{E}\left[\min\left\{\hat{\mu}_{a_{K}^*}^B, \hat{\mu}^*_{SE}\right\}\right] \geq \mathbb{E}\left[\min\left\{\hat{\mu}_{DE}^*, \hat{\mu}^*_{SE}\right\}\right]=\mathbb{E}\left[\hat{\mu}_{CDE}^*\right], 1 \leq K \leq N$.
\end{proof}

\section{B. Proof of Convergence of Action Candidate Based Clipped Double Q-learning}
For current variants of Double Q-learning, there are two main updating methods including random updating and simultaneous updating. In former method, only one Q-function is updated while in latter method, we update both them with the same target value. In this section, we prove that our action candidate based clipped Double Q-learning can converge to the optimal action value for both updaing methods under  finite MDP setting.

\subsection{B.1 Convergence Analysis on Random Updating}
In our action candidate based clipped Double Q-learning (see Algorithm 1 in the paper), we randomly choose one Q-function to update its action value in each time step.
Specifically, with collected experience $\left\langle s_t, a_t, r_t, s_{t+1}\right\rangle$, if we update $Q^A$, the updating formula is shown as below:
\begin{normalsize}
	\begin{equation} \label{policy}
	\begin{split}
	Q_{t+1}^A\left(s_t,a_t\right) \leftarrow Q_t^A\left(s_t,a_t\right) + \alpha_t\left(s_t,a_t\right) \left( r_t+\gamma \min\left\{Q_t^B\left(s_{t+1}, a_K^*\right), Q_t^A\left(s_{t+1}, a^*\right)\right\} - Q_t^A\left(s_t,a_t\right)\right),
	\end{split}
	\end{equation}
\end{normalsize}
where $a_K^*=\arg\max_{a\in\mathcal{M}_K}Q_t^A\left(s_{t+1}, a\right)$ with $\mathcal{M}_K=\left\{a_i\mid Q_t^B\left(s_{t+1}, a_i\right) \in {\rm top\ } K {\rm \ values\ in\ } Q_t^B\left(s_{t+1}, \cdot\right) \right\}$ and $a^* = \arg\max_a Q^A\left(s_{t+1}, a\right)$. Instead, if we update $Q_t^B$, the updating formula is:
\begin{normalsize}
	\begin{equation} \label{policy}
	\begin{split}
	Q_{t+1}^B\left(s_t,a_t\right) \leftarrow Q_t^B\left(s_t,a_t\right) + \alpha_t\left(s_t,a_t\right) \left( r_t+\gamma \min\left\{Q_t^A\left(s_{t+1}, b_K^*\right), Q_t^B\left(s_{t+1}, b^*\right)\right\} - Q_t^B\left(s_t,a_t\right)\right),
	\end{split}
	\end{equation}
\end{normalsize}
where $b_K^*=\arg\max_{a\in\mathcal{M}_K}Q_t^B\left(s_{t+1}, a\right)$ with $\mathcal{M}_K=\left\{a_i\mid Q_t^A\left(s_{t+1}, a_i\right) \in {\rm top\ } K \ {\rm values\ in\ } Q_t^A\left(s_{t+1}, \cdot\right) \right\}$ and $b^* = \arg\max_a Q_t^B\left(s_{t+1}, a\right)$. Next, we prove that our clipped Double Q-learning can converge to the optimal Q-function $Q^*\left(s,a\right)$ under the updating method above.

\begin{lemma}
	Consider a stochastic process $\left(\zeta_{t}, \Delta_{t}, F_{t}\right)$, $t \geq 0$, where $\zeta_{t}$, $\Delta_{t}$, $F_{t}: X \rightarrow \mathbb{R}$ satisfy the equations:
	\begin{normalsize}
		\begin{equation} \label{policy}
		\begin{split}
		\Delta_{t+1}\left(x_{t}\right)=\left(1-\zeta_{t}\left(x_{t}\right)\right) \Delta_{t}\left(x_{t}\right)+\zeta_{t}\left(x_{t}\right) F_{t}\left(x_{t}\right),
		\end{split}
		\end{equation}
	\end{normalsize}
	where $x_{t} \in X$ and $t=0,1,2,\ldots$. Let $P_t$ be a sequence of increasing $\sigma$-fields such that $\zeta_{0}$ and $\Delta_0$ are $P_{0}$-measurable and $\zeta_t$, $\Delta_t$ and $F_{t-1}$ are $P_t$-measure, $t=1, 2, \ldots$. Assume that the following hold: 
	
	1) The set $X$ is finite. 
	
	2) $\zeta_t(x_t) \in [0,1]$, $\sum_{t} \zeta_{t}\left(x_{t}\right)=\infty$, $\sum_{t}\left(\zeta_{t}\left(x_{t}\right)\right)^{2}<\infty$ with probability 1 and $\forall x \neq x_{t}: \zeta_{t}(x)=0$. 
	
	3) $\| \mathbb{E}\left[F_{t} \mid P_{t}\right] \| \leq \kappa\| \Delta_{t}\|+c_{t}$, where $\kappa \in [0,1)$ and $c_t$ converges to zero with probability 1. 
	
	4) $\operatorname{Var}\left[F_{t}\left(x_{t}\right) \mid P_{t}\right] \leq K\left(1+\kappa|| \Delta_{t}||\right)^{2}$, where $K$ is some constant. Here $\|\cdot\|$ denotes a maximum norm. 
	
	Then $\Delta_t$ converges to zero with probability 1.
\end{lemma}
\begin{theorem}
	Given the following conditions: 
	
	1) Each state action pair is sampled an infinite number of times. 
	
	2) The MDP is finite, that is $|S\times A|< \infty$.
	
	3) $\gamma \in [0, 1)$.
	
	4) Q values are stored in a lookup table.
	
	5) Both $Q^A$ and $Q^B$ receive an infinite number of updates.
	
	6) The learning rates satisfy $\alpha_{t}(s, a) \in[0,1], \sum_{t} \alpha_{t}(s, a)=\infty, \sum_{t}\left(\alpha_{t}(s, a)\right)^{2}<\infty$ with probability 1 and $\alpha_t(s,a)=0, \forall(s, a) \neq\left(s_{t}, a_{t}\right).$
	
	7) $\operatorname{Var}[r(s, a)]<\infty, \forall s, a$.
	

	Then, our proposed action candidate based clipped Double Q-learning under random updating will converge to the optimal value function $Q^*$ with probability 1.
\end{theorem}

\begin{proof}
	We apply Lemma 1 with $P_{t}=\left\{Q_{0}^{A}, Q_{0}^{B}, s_{0}, a_{0}, \alpha_{0}, r_{1}, s_{1}, \ldots, s_{t}, a_{t}\right\}$, $X=S \times A$, $\Delta_{t}=Q_{t}^{A}-Q^*$, $\zeta_{t}=\alpha_{t}$ and $F_t\left(s_t,a_t\right)=r_t+\gamma\min\left\{Q_t^B(s_{t+1},a_K^*), Q_t^A(s_{t+1}, a^*)\right\}-Q^*\left(s_t, a_t\right)$, where $a^*=\arg \max_a Q_t^A(s_{t+1},a)$ and $a_K^*=\rm{arg}\max_{a\in\boldsymbol{\mathcal{M}}_K}Q^A(s_{t+1}, a)$.
	The conditions 1 and 4 of  Lemma 1 can hold by the conditions 2 and 7 of Theorem 1, respectively. 
	Condition 2 in Lemma 1 holds by the condition 6 in Theorem 2 along with our selection of $\zeta_t=\alpha_t$.
	
	Then, we just need verify the condition 3 on the expected condition of $F_t$ holds. We can write:
	\begin{normalsize}
		\begin{equation} \label{policy}
		\begin{split}
		F_t\left(s_t,a_t\right)&=r_t+\gamma\min\left\{Q_t^B(s_{t+1},a_K^*), Q_t^A(s_{t+1}, a^*)\right\}-Q^*\left(s_t, a_t\right) + \gamma Q_t^A\left(s_t, a_t\right) - \gamma Q_t^A\left(s_t, a_t\right)\\
		&=r_t + \gamma Q_t^A(s_t, a_t) - Q^*\left(s_t, a_t\right) + \gamma\min\left\{Q_t^B(s_{t+1},a_K^*), Q_t^A(s_{t+1}, a^*)\right\} - \gamma Q_t^A\left(s_t, a_t\right)\\
		&=F_t^Q\left(s_t, a_t\right) + \gamma\min\left\{Q_t^B(s_{t+1},a_K^*), Q_t^A(s_{t+1}, a^*)\right\} - \gamma Q_t^A\left(s_t, a_t\right),
		\end{split}
		\end{equation}
	\end{normalsize}
	where $F_t^Q = r_t + \gamma Q_t^A(s_t, a_t) - Q^*\left(s_t, a_t\right)$ is the value of $F_t$ if normal Q-learning would be under consideration. It is wll-known that $\mathbb{E}\left[F_t^Q|P_t\right] \leq \gamma\|\Delta_t\|$, so in order to apply the lemma we identify $c_t=\gamma\min\left\{Q_t^B(s_{t+1},a_K^*), Q_t^A(s_{t+1}, a^*)\right\} - \gamma Q_t^A\left(s_t, a_t\right)$ and it suffices to show that $\Delta_t^{BA}=Q_t^B - Q_t^A$ converges to to zero. Depending on whether $Q^B$ or $Q^A$ is updated, the update of $\Delta_t^{BA}$ at time $t$ is either
	\begin{normalsize}
		\begin{equation} \label{policy}
		\begin{split}
		\Delta_{t+1}^{B A}\left(s_{t}, a_{t}\right) &=\Delta_{t}^{B A}\left(s_{t}, a_{t}\right)+\alpha_{t}\left(s_{t}, a_{t}\right) F_{t}^{B}\left(s_{t}, a_{t}\right), \text { or } \\ \Delta_{t+1}^{B A}\left(s_{t}, a_{t}\right) &=\Delta_{t}^{B A}\left(s_{t}, a_{t}\right)-\alpha_{t}\left(s_{t}, a_{t}\right) F_{t}^{A}\left(s_{t}, a_{t}\right),
		\end{split}
		\end{equation}
	\end{normalsize}
	where $F_t^A\left(s_t, a_t\right)=r_t+\gamma\min\left\{Q_t^B(s_{t+1},a_K^*), Q_t^A(s_{t+1}, a^*)\right\}-Q_t^A\left(s_t, a_t\right)$ and $F_t^B\left(s_t, a_t\right)=r_t+\gamma\min\left\{Q_t^A(s_{t+1},b_K^*), Q_t^B(s_{t+1}, b^*)\right\}-Q_t^B\left(s_t, a_t\right)$. We define $\zeta_{t}^{B A}=\frac{1}{2} \alpha_{t}$. Then,
	\begin{normalsize}
		\begin{equation}
		\begin{split}
		\mathbb{E}\left[\Delta_{t+1}^{B A}\left(s_{t}, a_{t}\right) \mid P_{t}\right] &=\Delta_{t}^{B A}\left(s_{t}, a_{t}\right)+\mathbb{E}\left[\alpha_{t}\left(s_{t}, a_{t}\right) F_{t}^{B}\left(s_{t}, a_{t}\right)-\alpha_{t}\left(s_{t}, a_{t}\right) F_{t}^{A}\left(s_{t}, a_{t}\right) \mid P_{t}\right] \\ &=\left(1-\zeta_{t}^{B A}\left(s_{t}, a_{t}\right)\right) \Delta_{t}^{B A}\left(s_{t}, a_{t}\right)+\zeta_{t}^{B A}\left(s_{t}, a_{t}\right) \mathbb{E}\left[F_{t}^{B A}\left(s_{t}, a_{t}\right) \mid P_{t}\right],
		\end{split}
		\end{equation}
	\end{normalsize}
	where $\mathbb{E}\left[F_t^{BA}\left(s_t, a_t\right)\mid P_t\right]=\gamma\mathbb{E}\left[\min\left\{Q_t^A(s_{t+1},b_K^*), Q_t^B(s_{t+1}, b^*)\right\}-\min\left\{Q_t^B(s_{t+1},a_K^*), Q_t^A(s_{t+1}, a^*)\right\} \mid P_t \right]$. For this step it is important that the selection whether to update $Q^A$ or $Q^B$ is independent on the sample (e.g. random).
	
	Assume $\mathbb{E}\left[\min\left\{Q_t^A(s_{t+1},b_K^*), Q_t^B(s_{t+1}, b^*)\right\}\right]\geq\mathbb{E}\left[\min\left\{Q_t^B(s_{t+1},a_K^*), Q_t^A(s_{t+1}, a^*)\right\} \mid P_t \right]$. Then,
	\begin{normalsize}
		\begin{equation}
		\begin{split}
		&\left|\mathbb{E}\left[F_t^{BA}\left(s_t, a_t\right)\mid P_t\right]\right|=
		\gamma\mathbb{E}\left[\min\left\{Q_t^A(s_{t+1},b_K^*), Q_t^B(s_{t+1}, b^*)\right\}-\min\left\{Q_t^B(s_{t+1},a_K^*), Q_t^A(s_{t+1}, a^*)\right\}\mid P_t\right]\\
		\leq & \gamma\mathbb{E}\left[\min\left\{Q_t^A(s_{t+1},a^*), Q_t^B(s_{t+1}, b^*)\right\}-\min\left\{Q_t^B(s_{t+1},a^*), Q_t^A(s_{t+1}, a^*)\right\}\mid P_t\right]\\
		\leq & \gamma\mathbb{E}\left[Q_t^A\left(s_{t+1}, a^*\right) \mid P_t\right] - \gamma\mathbb{E}\left[\min\left\{Q_t^B(s_{t+1},a^*), Q_t^A(s_{t+1}, a^*)\right\} \mid P_t\right]\\
		= & \gamma P\left(Q_t^B\left(s_{t+1}, a^*\right) \geq Q_t^A\left(s_{t+1}, a^*\right) \mid P_t \right) \mathbb{E}\left[Q_t^A(s_{t+1}, a^*)\mid Q_t^B\left(s_{t+1}, a^*\right) \geq Q_t^A\left(s_{t+1}, a^*\right), P_t \right]\\
		& + \gamma P\left(Q_t^B\left(s_{t+1}, a^*\right) < Q_t^A\left(s_{t+1}, a^*\right) \mid P_t \right) \mathbb{E}\left[Q_t^A(s_{t+1}, a^*)\mid Q_t^B\left(s_{t+1}, a^*\right) < Q_t^A\left(s_{t+1}, a^*\right), P_t \right]\\
		& - \gamma P\left(Q_t^B\left(s_{t+1}, a^*\right) \geq Q_t^A\left(s_{t+1}, a^*\right) \mid P_t \right) \mathbb{E}\left[Q_t^A(s_{t+1}, a^*)\mid Q_t^B\left(s_{t+1}, a^*\right) \geq Q_t^A\left(s_{t+1}, a^*\right), P_t \right]\\
		& - \gamma P\left(Q_t^B\left(s_{t+1}, a^*\right) < Q_t^A\left(s_{t+1}, a^*\right) \mid P_t \right) \mathbb{E}\left[Q_t^B(s_{t+1}, a^*)\mid Q_t^B\left(s_{t+1}, a^*\right) < Q_t^A\left(s_{t+1}, a^*\right), P_t \right]\\
		= & \gamma P\left(Q_t^B\left(s_{t+1}, a^*\right) < Q_t^A\left(s_{t+1}, a^*\right) \mid P_t \right)\mathbb{E}\left[\underbrace{Q_t^A(s_{t+1}, a^*)-Q_t^B(s_{t+1}, a^*)}_{\leq\|\Delta_t^{BA}\|} \mid Q_t^B\left(s_{t+1}, a^*\right) < Q_t^A\left(s_{t+1}, a^*\right), P_t\right]\\
		\leq & \gamma P\left(Q_t^B\left(s_{t+1}, a^*\right) < Q_t^A\left(s_{t+1}, a^*\right) \mid P_t \right)\|\Delta_t^{BA}\|\leq \gamma \|\Delta_t^{BA}\|,
		\end{split}
		\end{equation}
	\end{normalsize}
	where the first inequality is based on the monotonicity in Theorem 1. From Theorem 1, we have that the expected value of $\min\left\{Q_t^A(s_{t+1},b_K^*), Q_t^B(s_{t+1}, b^*)\right\}-\min\left\{Q_t^B(s_{t+1},a_K^*), Q_t^A(s_{t+1}, a^*)\right\}$ is no more than the one of $\min\left\{Q_t^A(s_{t+1},b_1^*), Q_t^B(s_{t+1}, b^*)\right\}-\min\left\{Q_t^B(s_{t+1},a_N^*), Q_t^A(s_{t+1}, a^*)\right\}$. Since $a^*=b_1^*=a_N^*$, we can have the first inequality above.
	The second inequality is based on that since $\min\left\{Q_t^A(s_{t+1},a^*), Q_t^B(s_{t+1}, b^*)\right\}$ is no more than $Q_t^A(s_{t+1},a^*)$, the expected value of the former is also no more than the one of latter.
	
	Now assume $\mathbb{E}\left[\min\left\{Q_t^A(s_{t+1},b_K^*), Q_t^B(s_{t+1}, b^*)\right\}\right] <\mathbb{E}\left[\min\left\{Q_t^B(s_{t+1},a_K^*), Q_t^A(s_{t+1}, a^*)\right\} \mid P_t \right]$ and therefore
	\begin{normalsize}
		\begin{equation}
		\begin{split}
		&\left|\mathbb{E}\left[F_t^{BA}\left(s_t, a_t\right)\mid P_t\right]\right|=
		\gamma\mathbb{E}\left[\min\left\{Q_t^B(s_{t+1},a_K^*), Q_t^A(s_{t+1}, a^*)\right\} - \min\left\{Q_t^A(s_{t+1},b_K^*), Q_t^B(s_{t+1}, b^*)\right\}\mid P_t\right]\\
		\leq & \gamma\mathbb{E}\left[\min\left\{Q_t^B(s_{t+1},b^*), Q_t^A(s_{t+1}, a^*)\right\} - \min\left\{Q_t^A(s_{t+1},b^*), Q_t^B(s_{t+1}, b^*)\right\}\mid P_t\right]\\
		\leq & \gamma\mathbb{E}\left[Q_t^B\left(s_{t+1}, b^*\right) \mid P_t\right] - \gamma\mathbb{E}\left[\min\left\{Q_t^A(s_{t+1},b^*), Q_t^B(s_{t+1}, b^*)\right\} \mid P_t\right]\\
		= & \gamma P\left(Q_t^B\left(s_{t+1}, b^*\right) \geq Q_t^A\left(s_{t+1}, b^*\right) \mid P_t \right) \mathbb{E}\left[Q_t^B(s_{t+1}, b^*)\mid Q_t^B\left(s_{t+1}, b^*\right) \geq Q_t^A\left(s_{t+1}, b^*\right), P_t \right]\\
		& +  \gamma P\left(Q_t^B\left(s_{t+1}, b^*\right) < Q_t^A\left(s_{t+1}, b^*\right) \mid P_t \right) \mathbb{E}\left[Q_t^B(s_{t+1}, b^*)\mid Q_t^B\left(s_{t+1}, b^*\right) < Q_t^A\left(s_{t+1}, b^*\right), P_t \right]\\
		& - \gamma P\left(Q_t^B\left(s_{t+1}, b^*\right) \geq Q_t^A\left(s_{t+1}, b^*\right) \mid P_t \right) \mathbb{E}\left[Q_t^A(s_{t+1}, b^*)\mid Q_t^B\left(s_{t+1}, b^*\right) \geq Q_t^A\left(s_{t+1}, b^*\right), P_t \right]\\
		& - \gamma P\left(Q_t^B\left(s_{t+1}, b^*\right) < Q_t^A\left(s_{t+1}, b^*\right) \mid P_t \right) \mathbb{E}\left[Q_t^B(s_{t+1}, b^*)\mid Q_t^B\left(s_{t+1}, b^*\right) < Q_t^A\left(s_{t+1}, b^*\right), P_t \right]\\
		= & \gamma P\left(Q_t^B\left(s_{t+1}, b^*\right) \geq Q_t^A\left(s_{t+1}, b^*\right) \mid P_t \right)\mathbb{E}\left[\underbrace{Q_t^B(s_{t+1}, b^*)-Q_t^A(s_{t+1}, b^*)}_{\leq\|\Delta_t^{BA}\|} \mid Q_t^B\left(s_{t+1}, b^*\right) \geq Q_t^A\left(s_{t+1}, b^*\right), P_t\right]\\
		\leq & \gamma P\left(Q_t^B\left(s_{t+1}, a^*\right) \geq Q_t^A\left(s_{t+1}, a^*\right) \mid P_t \right)\|\Delta_t^{BA}\|\leq \gamma \|\Delta_t^{BA}\|,
		\end{split}
		\end{equation}
	\end{normalsize}
	where the first inequality is based on the monotonicity in Theorem 1. From Theorem 1, we have the expected value of $\min\left\{Q_t^B(s_{t+1},a_K^*), Q_t^A(s_{t+1}, a^*)\right\} - \min\left\{Q_t^A(s_{t+1},b_K^*), Q_t^B(s_{t+1}, b^*)\right\}$ is no more than the one of $\min\left\{Q_t^B(s_{t+1},a_1^*), Q_t^A(s_{t+1}, a^*)\right\} - \min\left\{Q_t^A(s_{t+1},b_N^*), Q_t^B(s_{t+1}, b^*)\right\}$. Since $b^*=a_1^*=b_N^*$, we can have the first inequality above.
	The second inequality is based on that since $\min\left\{Q_t^B(s_{t+1},b^*), Q_t^A(s_{t+1}, a^*)\right\}$ is no more than $Q_t^B(s_{t+1},b^*)$, the expected value of the former is also no more than the one of latter.

	Clearly, one of the assumptions must hold at each time step and in both cases we obtain the desired result that $\left|\mathbb{E}\left[F_t^{BA} \mid P_t\right]\right| \leq \gamma \|\Delta_t^{BA}\|$. Applying the lemma yields convergence of $\Delta_t^{BA}$ to zero, which in turn ensures that the original process also converges in the limit.
\end{proof}

\subsection{B.2 Convergence Analysis on Simultaneous Updating}
In Algorithm 2 of the paper, we update our two Q-functions with the same target value in each time step. In this section, we further prove that our action candidate based clipped Double Q-learning can also converge to the optimal Q-function $Q^*\left(s,a\right)$ under such updating method.

Specificially, with the collected experience $\langle s_t, a_t, r_t, s_{t+1}\rangle$, we set the target value $y_t$ as below:
\begin{normalsize}
	\begin{equation} \label{policy}
	\begin{split}
	y_t = r_t + \gamma\min\left\{Q^B\left(s_{t+1}, a_K^*\right), Q^A\left(s_{t+1}, a^*\right)\right\},
	\end{split}
	\end{equation}
\end{normalsize}
where $a_K^*=\arg\max_{a\in\mathcal{M}_K}Q^A\left(s_{t+1}, a\right)$ with $\mathcal{M}_K=\left\{a_i\mid Q^B\left(s_{t+1}, a_i\right) \in {\rm top\ } K {\rm\ values\ in\ } Q^B\left(s_{t+1}, \cdot\right) \right\}$ and $a^* = \arg\max_a Q^A\left(s_{t+1}, a\right)$. Then, both Q-functions are updated as below:
\begin{normalsize}
	\begin{equation} \label{policy}
	\begin{split}
	&Q_{t+1}^A\left(s_t, a_t\right)\leftarrow Q_t^A\left(s_t, a_t\right) + \alpha_t\left(s_t, a_t\right)\left(y_t - Q_t^A\left(s_t, a_t\right)\right)\\
	&Q_{t+1}^B\left(s_t, a_t\right)\leftarrow Q_t^B\left(s_t, a_t\right) + \alpha_t\left(s_t, a_t\right)\left(y_t - Q_t^B\left(s_t, a_t\right)\right)
	\end{split}
	\end{equation}
\end{normalsize}
\begin{theorem}
	Given the following conditions: 
	
	1) Each state action pair is sampled an infinite number of times. 
	
	2) The MDP is finite, that is $|S\times A|< \infty$.
	
	3) $\gamma \in [0, 1)$.
	
	4) Q values are stored in a lookup table.
	
	5) Both $Q^A$ and $Q^B$ receive an infinite number of updates.
	
	6) The learning rates satisfy $\alpha_{t}(s, a) \in[0,1], \sum_{t} \alpha_{t}(s, a)=\infty, \sum_{t}\left(\alpha_{t}(s, a)\right)^{2}<\infty$ with probability 1 and $\alpha_t(s,a)=0, \forall(s, a) \neq\left(s_{t}, a_{t}\right).$
	
	7) $\operatorname{Var}[r(s, a)]<\infty, \forall s, a$.
	
	
	Then, our proposed action candidate based clipped Double Q-learning under simultaneous updating will converge to the optimal value function $Q^*$ with probability 1.
\end{theorem}
\begin{proof}
	We apply Lemma 1 with $P_{t}=\left\{Q_{0}^{A}, Q_{0}^{B}, s_{0}, a_{0}, \alpha_{0}, r_{1}, s_{1}, \ldots, s_{t}, a_{t}\right\}$, $X=S \times A$, $\Delta_{t}=Q_{t}^{A}-Q^*$, $\zeta_{t}=\alpha_{t}$ and $F_t\left(s_t,a_t\right)=r_t+\gamma\min\left\{Q_t^B(s_{t+1},a_K^*), Q_t^A(s_{t+1}, a^*)\right\}-Q^*\left(s_t, a_t\right)$, where $a^*=\arg \max_a Q_t^A(s_{t+1},a)$ and $a_K^*=\rm{arg}\max_{a\in\boldsymbol{\mathcal{M}}_K}Q^A(s_{t+1}, a)$.
	The conditions 1 and 4 of  Lemma 1 can hold by the conditions 2 and 7 of Theorem 3, respectively. 
	Condition 2 in Lemma 1 holds by the condition 6 in Theorem 3 along with our selection of $\zeta_t=\alpha_t$. Further, we have
	\begin{normalsize}
		\begin{equation} \label{policy}
		\begin{split}
		\Delta_{t+1}\left(s_t, a_t\right)&=\left(1-\alpha_t\left(s_t, a_t\right)\right)\left(Q_t^A\left(s_t, a_t\right)-Q^*\left(s_t, a_t\right)\right)+\alpha_t\left(s_t, a_t\right)\left(y_t-Q^*\left(s_t, a_t\right)\right)\\
		& = \left(1-\alpha_t\left(s_t, a_t\right)\right)\Delta_t\left(s_t, a_t\right) + \alpha_t\left(s_t, a_t\right)F_t\left(s_t, a_t\right),
		\end{split}
		\end{equation}
	\end{normalsize}
	where we have defined $F_t\left(s_t, a_t\right)$ as:
	\begin{normalsize}
		\begin{equation} \label{policy}
		\begin{split}
		F_t\left(s_t, a_t\right) &= y_t - Q_t^*\left(s_t, a_t\right) 
		= y_t - Q_t^*\left(s_t, a_t\right) + \gamma Q_t^A\left(s_{t+1}, a^*\right) - \gamma Q_t^A\left(s_{t+1}, a^*\right)\\
		&=F_t^Q(s_t, a_t) + c_t,
		\end{split}
		\end{equation}
	\end{normalsize}
	where $F_t^Q=r_t + \gamma Q_t^A\left(s_{t+1}, a^*\right) - Q_t^*\left(s_t, a_t\right)$ denotes the value of $F_t$ under standard Q-learing and 
	\begin{normalsize}
		\begin{equation} \label{policy}
		\begin{split}
		c_t=\gamma\min\left\{Q^B\left(s_{t+1}, a_K^*\right), Q^A\left(s_{t+1}, a^*\right)\right\}-\gamma Q_t^A\left(s_{t+1}, a^*\right).
		\end{split}
		\end{equation}
	\end{normalsize}
	As $\mathbb{E}\left[F_t^Q|P_t\right]\leq\gamma\|\Delta_t\|$ is a well-known result, condition 3 of Lemma 1 holds if it can be shown that $c_t$ converges to 0 with probability 1.
	Further, due to $\min\left\{Q^B\left(s_{t+1}, a_K^*\right), Q^A\left(s_{t+1}, a^*\right)\right\}$ is no more than $Q^B\left(s_{t+1}, a_K^*\right)$ and $Q^B\left(s_{t+1}, a_K^*\right)$ is also no more than $Q^B\left(s_{t+1}, a_1^*\right)$ (based on the Property 1), we can have $\min\left\{Q^B\left(s_{t+1}, a_K^*\right), Q^A\left(s_{t+1}, a^*\right)\right\} \leq Q^B\left(s_{t+1}, a_1^*\right)$. Since $a_1^*=b^*$, we can have $Q^B\left(s_{t+1}, a_1^*\right)=Q^B\left(s_{t+1}, b^*\right)$. Finally, due to $Q_t^A\left(s_{t+1}, a^*\right) \geq Q_t^A\left(s_{t+1}, b^*\right)$, we can know that:
	\begin{normalsize}
		\begin{equation} \label{policy}
		\begin{split}
		c_t=\gamma\min\left\{Q^B\left(s_{t+1}, a_K^*\right), Q^A\left(s_{t+1}, a^*\right)\right\}-\gamma Q_t^A\left(s_{t+1}, a^*\right)\leq \gamma Q^B\left(s_{t+1}, b^*\right)- \gamma Q^A\left(s_{t+1}, b^*\right).
		\end{split}
		\end{equation}
	\end{normalsize}
	Thus, $c_t$ converges to $0$ if $\Delta_t^{BA}$ converges to $0$ where $\Delta_{t}^{B A}\left(s_{t}, a_{t}\right)=Q_{t}^{B}\left(s_{t}, a_{t}\right)-Q_{t}^{A}\left(s_{t}, a_{t}\right)$.
	The update of $\Delta_t^{BA}$ at time $t$ is the sum of updates of $Q^A$ and $Q^B$:
	\begin{normalsize}
		\begin{equation} \label{policy}
		\begin{split}
		\Delta_{t+1}^{B A}\left(s_{t}, a_{t}\right) &=\Delta_{t}^{B A}\left(s_{t}, a_{t}\right)+\alpha_{t}\left(s_{t}, a_{t}\right)\left(y_t-Q_{t}^{B}\left(s_{t}, a_{t}\right)-\left(y_t-Q_{t}^{A}\left(s_{t}, a_{t}\right)\right)\right) \\ &=\Delta_{t}^{B A}\left(s_{t}, a_{t}\right)+\alpha_{t}\left(s_{t}, a_{t}\right)\left(Q_{t}^{A}\left(s_{t}, a_{t}\right)-Q_{t}^{B}\left(s_{t}, a_{t}\right)\right) \\ &=\left(1-\alpha_{t}\left(s_{t}, a_{t}\right)\right) \Delta_{t}^{B A}\left(s_{t}, a_{t}\right).
		\end{split}
		\end{equation}
	\end{normalsize}
	Clearly, $\Delta_t^{BA}$ converges to $0$, which then shows we have satisfied condition 3 of Lemma , which implies that $Q^A\left(s_t, a_t\right)$ converges to $Q^*\left(s_t, a_t\right)$. Similarly, wo get the convergence of $Q^B\left(s_t, a_t\right)$ to the optimal value function by choosing $\Delta_t=Q_t^B-Q^*$ and repeating the same arguments, thus proving Theorem 3.
\end{proof}

\section{C. Additional Experimental Results}
In this section, we provide some additional experimental results on Grid World, MinAtar and MuJoCo tasks.

\subsubsection{Grid World} 
In Grid World environment, for action candidate based clipped Double Q-learning (AC-CDQ), we further evaluate its performance on the grid world game with size $3 \times 3$ and $4 \times 4$. As shown in Fig~\ref{fig:gridworld}, benefiting from  the precise estimation about the optimal action value (closest to the dash line), AC-CDQ ($K=2$) presents the superior performance. 

\subsubsection{MinAtar} 
In MinAtar games, for action candidate based clipped Double DQN (AC-CDDQN), we additionally list the learning curves about the averaged reward and the estimated maximum action value with different numbers of the action candidates ($K=\left\{2, 3, 4\right\}$). As shown in Fig~\ref{fig:minatari_varyk}, except for the case that $K=4$ in Breakout game, our method can consistently perform better than clipped Double DQN. Moreover, as shown in the two plots on the right, our deep version can effectively balance the overestimated DQN and the underestimated clipped Double DQN. Further, it also
empirically follows the monotonicity in Theorem 1, that is as the number $K$ of action candidates decreases, the underestimation bias in clipped Double DQN reduces monotonically.

\subsubsection{MuJoCo Tasks}
In MuJoCo tasks, for action candidate based TD3 (AC-TD3), we provide the additional learning curves on Walker2D-v2, Pusher-v2, Hopper-v2 and Reacher-v2 in Fig~\ref{fig:mujoco}. Moreover, we also test the performance variance under different number of the action candidates ($K=\left\{32, 64, 128\right\}$) in Walker2D-v2, Pusher-v2, Swimmer-v2 and Ant-v2 in Fig~\ref{fig:mujoco_vary}. The plots show that AC-TD3 can consistently obtain the robust and superior performance with different action candidate sizes.

\begin{figure*}[ht]
	\centering
	\includegraphics[width=\textwidth]{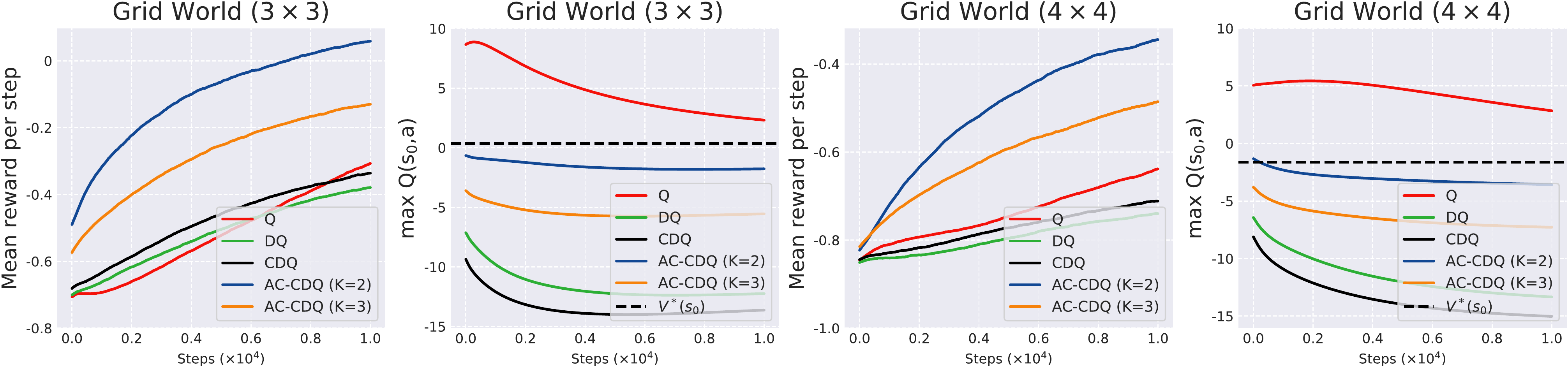}
	\vspace{-6mm}
	\caption{Learning curves about the mean reward per step and the estimated maximum action value from the state $s_0$ (the black dash line demotes the optimal state value $V^*(s_0)$). The results are averaged over 10000 experiments and each experiment contains 10000 steps. 
		We set the number of the action candidates to 2 and 3, respectively. Q: Q-learning, DQ: Double Q-learning, CDQ: clipped Double Q-learning.
	}	
	\label{fig:gridworld}
	\vspace{-5mm}
\end{figure*}

\begin{figure*}[ht]
	\centering
	\includegraphics[width=\textwidth]{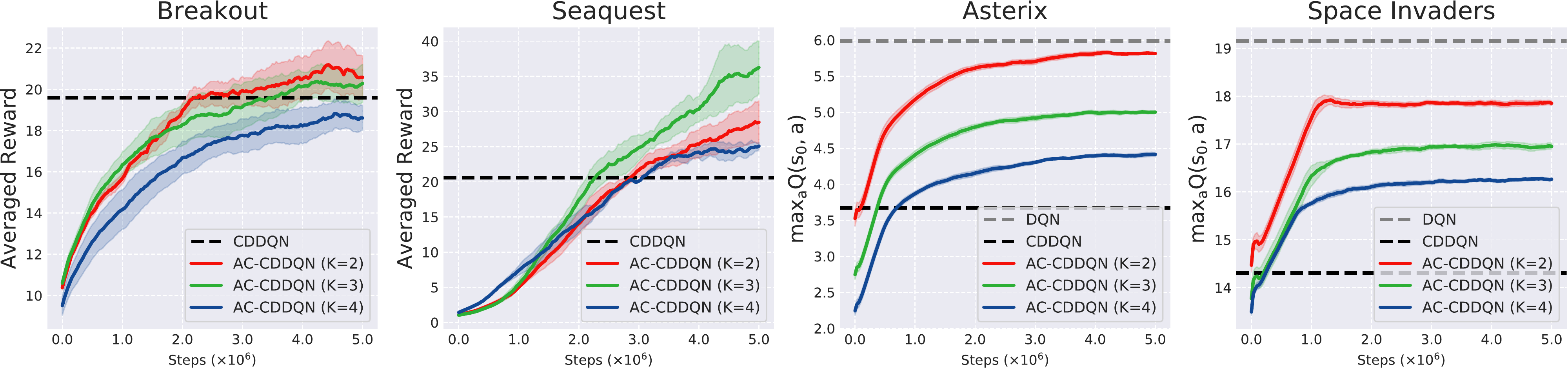}
	\vspace{-6mm}
	\caption{Learning curves about the averaged reward (two plots on the left) and estimated maximum action value (two plots on the right) for AC-CDDQN with different numbers of the action candidates (K=2, 3, 4). CDDQN: clipped Double DQN.
	}	
	\label{fig:minatari_varyk}
	\vspace{-3mm}
\end{figure*}

\begin{figure*}[ht]
	\centering
	\includegraphics[width=\textwidth]{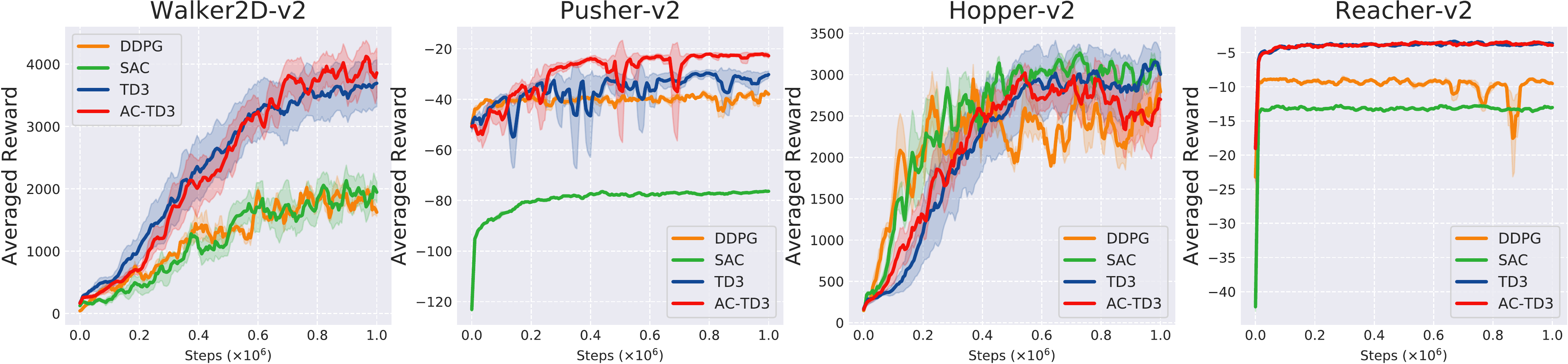}
	\vspace{-6mm}
	\caption{Learning curves for the OpenAI Gym continuous control tasks. The shaded region represents half a standard deviation of the average evaluation over 10 trials.
	}	
	\label{fig:mujoco}
	\vspace{-3mm}
\end{figure*}

\begin{figure*}[ht]
	\centering
	\includegraphics[width=\textwidth]{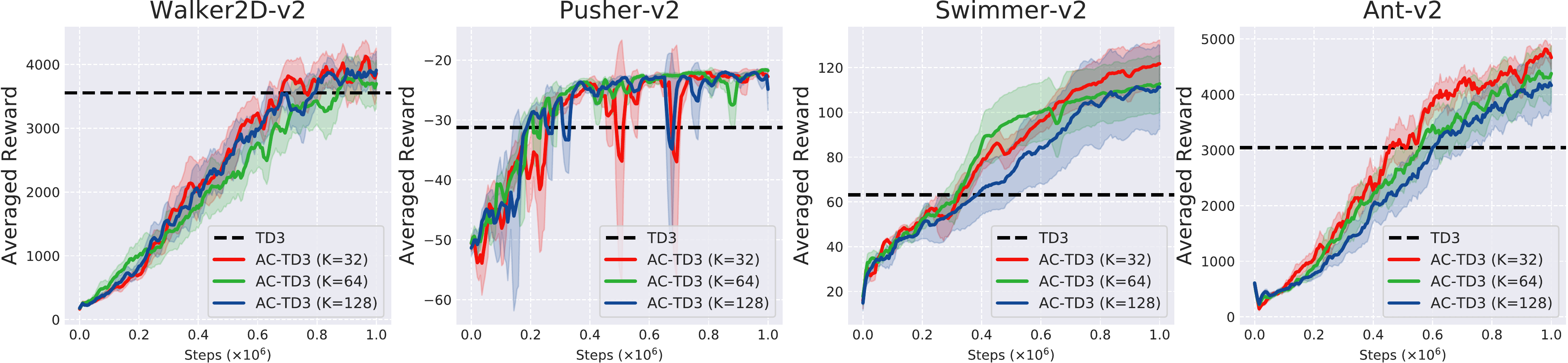}
	\vspace{-6mm}
	\caption{Learning curves for AC-TD3 with different numbers of the action candidates ($K=\left\{32, 64, 128\right\}$). The shaded region represents half a standard deviation of the average evaluation over 10 trials.
	}	
	\label{fig:mujoco_vary}
	\vspace{-3mm}
\end{figure*}

\section{D. Hyper-parameters Setting}
\subsubsection{Action Candidate Based Clipped Double DQN}
In this method, the number of frames is $5\cdot10^6$; the discount factor is $0.99$; reward scaling is 1.0; the batch size is $32$; the buffer size is $1\cdot10^6$; the frequency of updating the target network is $1000$; the optimizer is RMSprop with learning $2.5\cdot10^{-4}$, squared gradient momentum $0.95$ and minimum squared gradient $0.01$; the iteration per time step is $1$.

\subsubsection{Action Candidate Based TD3}
In this method, the number of frames is $1\cdot10^6$; the discount factor is $0.99$; reward scaling is 1.0; the batch size is $256$; the buffer size is $1\cdot10^6$; the frequency of updating the target network is $2$; the optimizers for actor and critic are Adams with learning $3\cdot10^{-4}$; the iteration per time step is $1$. All experiments are conducted on a server with NVIDIA TITAN V.

\end{document}